\pdfoutput=1

\documentclass[11pt]{article}

\usepackage[final]{acl}

\usepackage{times}
\usepackage{latexsym}

\usepackage[T1]{fontenc}

\usepackage[utf8]{inputenc}

\usepackage{microtype}

\usepackage{inconsolata}

\usepackage{booktabs}
\usepackage{amsfonts}
\usepackage{caption}
\usepackage{subcaption}
\usepackage{graphicx}
\usepackage{wrapfig}
\usepackage{amssymb, amsmath, amsthm}
\usepackage{multirow}
\newtheorem{assumption}{Assumption}
\newtheorem{theorem}{Theorem}

\newtheorem{lemma}[theorem]{Lemma}
\newtheorem{proposition}{Proposition}
\usepackage{longtable}
\usepackage{xcolor}

\newcommand{\modelname}{LM-Steer}

\usepackage{xparse}
\NewDocumentCommand{\heng}
{ mO{} }{\textcolor{red}{\textsuperscript{\textit{Heng}}\textsf{\textbf{\small[#1]}}}}
\NewDocumentCommand{\hengsolved}
{ mO{} }{\textcolor{red}{\textsuperscript{\textit{Heng}}\textsf{\sout{\small[#1]}}}}
\NewDocumentCommand{\hengarchived}
{ mO{} }{}
\NewDocumentCommand{\chihan}
{ mO{} }{\textcolor{blue}{\textsuperscript{\textit{chi}}[#1]}}
\NewDocumentCommand{\xingyao}
{ mO{} }{\textcolor{orange}{\textsuperscript{\textit{Xingyao}}[#1]}}
\NewDocumentCommand{\yi}
{ mO{} }{\textcolor{brown}{\textsuperscript{\textit{yi}}[#1]}}

\NewDocumentCommand{\chenkai}
{ mO{} }{\textcolor{purple}{\textsuperscript{\textit{chenkai}}[#1]}}

\usepackage{enumitem}

\title{Word Embeddings Are Steers for Language Models}

\author{Chi Han, Jialiang Xu, Manling Li, Yi Fung, Chenkai Sun, \\ {\bf Nan Jiang, Tarek Abdelzaher, Heng Ji} \\
University of Illinois Urbana-Champaign \\
  \texttt{\{chihan3, jx17, manling2, yifung2, chenkai5} \\
  \texttt{nanjiang, zaher, hengji\}@illinois.edu} \\}

\begin{document}
\maketitle
\begin{abstract}
    Language models (LMs) automatically learn word embeddings during pre-training on language corpora.
Although word embeddings are usually interpreted as feature vectors for individual words, their roles in language model generation remain underexplored.
In this work, we theoretically and empirically revisit output word embeddings and find that their linear transformations are equivalent to steering language model generation styles.
We name such steers \modelname{}s and find them existing in LMs of all sizes.
It requires learning parameters equal to 0.2\% of the original LMs' size for steering each style.
On tasks such as language model detoxification and sentiment control, \modelname{}s can achieve comparable or superior performance compared with state-of-the-art controlled generation methods while maintaining a better balance with generation quality.
The learned \modelname{} serves as a lens in text styles: it reveals that word embeddings are interpretable when associated with language model generations and can highlight text spans that most indicate the style differences.
An \modelname{} is transferrable between different language models by an explicit-form calculation.
One can also continuously steer LMs simply by scaling the \modelname{} or compose multiple \modelname{}s by adding their transformations.
Our codes are publicly available at \url{https://github.com/Glaciohound/LM-Steer}.
\footnote{Please be advised that this paper contains potentially controversial results and examples to some readers, included solely for research purposes to explore model capabilities.}
\end{abstract}

\begin{figure}[t!]
\centering
\includegraphics[width=0.5\textwidth]{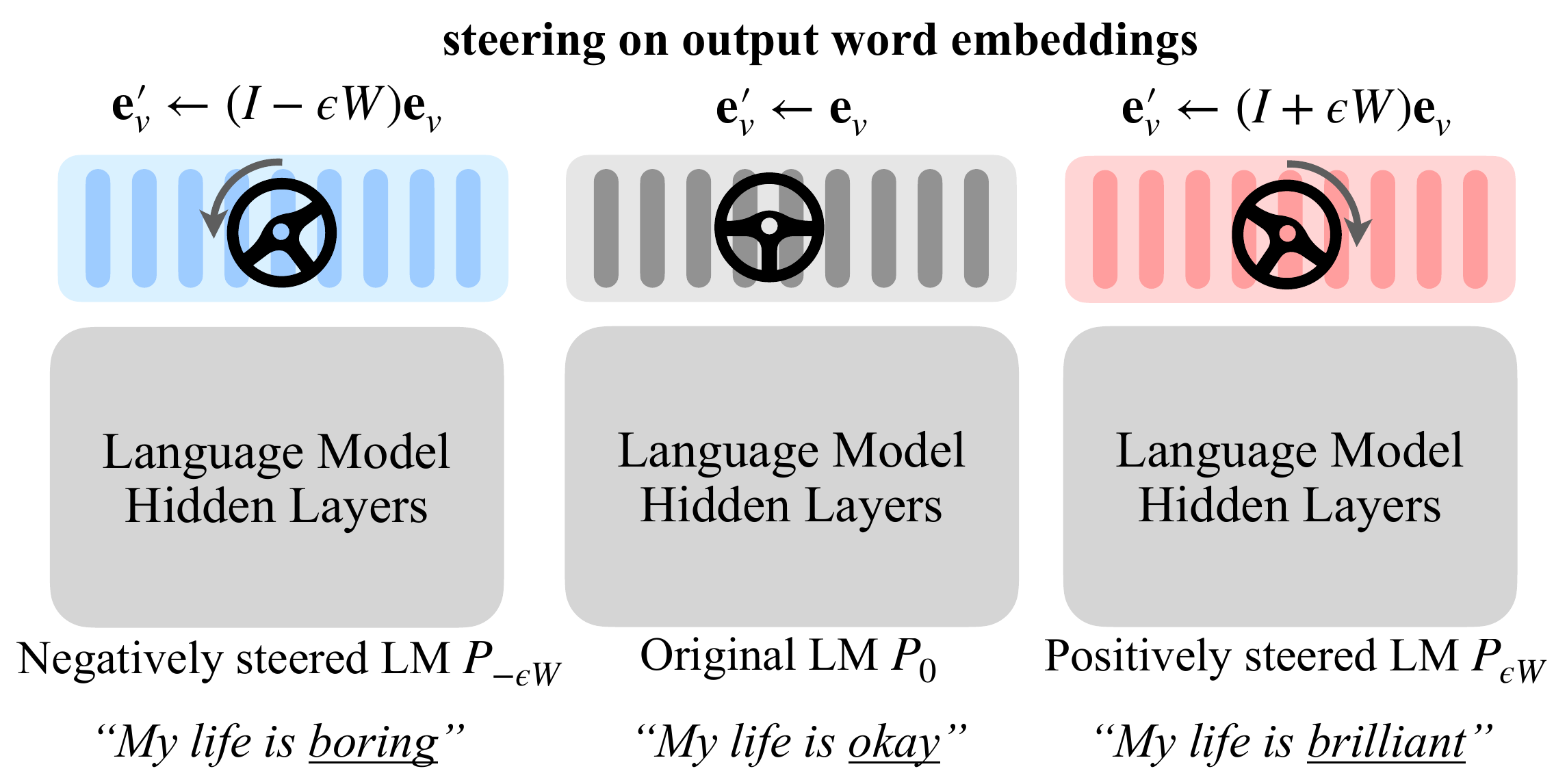}
\caption{We find hidden steers in \textbf{output word embeddings}. By linearly transforming word embeddings, language model generations are ``steered'' toward different style polarity and levels.}
\label{fig:teaser}
\end{figure}

\begin{figure*}[t!]
\centering
\includegraphics[width=\textwidth]{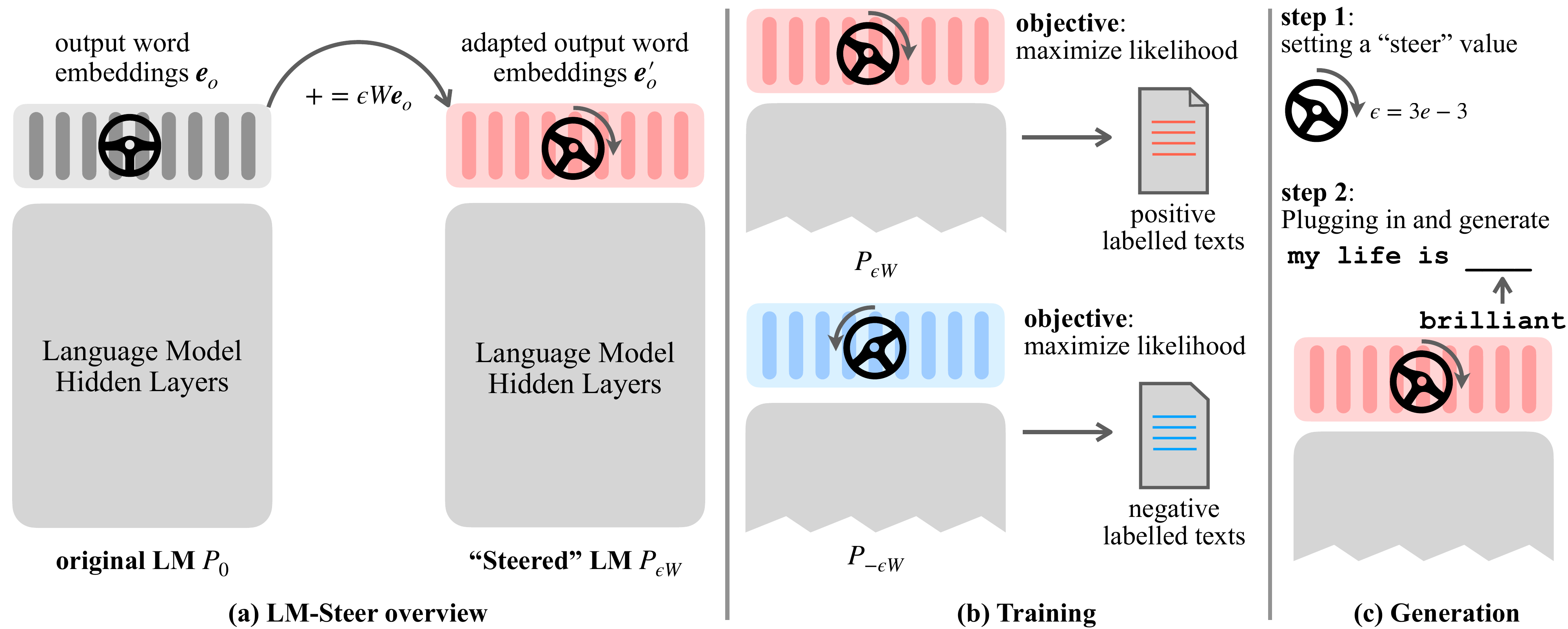}
\caption{An overview of \modelname{}. \textbf{(a)}: \modelname{} applies a linear factor $\epsilon W \mathbf{e}_v$ to each word embedding for language model conditioning. \textbf{(b)}: During training, we use a positively steered model $P_{\epsilon W}$ to maximize likelihood on positively labeled texts, and vise versa. \textbf{(c)}: For generation, one only needs to specify a steering value $\epsilon$, and then proceed with normal decoding.}
\label{fig:overview}
\end{figure*}
\section{Introduction}
\label{sec:introduction}

In recent years, language models (LMs) have significantly advanced various natural language processing (NLP) tasks such as machine translation, sentiment analysis, schema induction, summarization, and sociocultural understanding~\citep{brown2020language, kojimalarge,hierarchicalschema2023,radford2018improving, openai2023gpt4, fung-etal-2023-normsage, fung2024massively}.
Their \textbf{output word embeddings} are learned automatically to calculate word output likelihoods during pre-training on language corpora. Typically, the dot product $\mathbf{c}^\top\mathbf{e}_v$ between a computed context vector $\mathbb{c}$ and a learnable output word embedding $\mathbf{e}_v$ for token $v$ is usually used as the word logit. The word output probability is defined as the softmax over all word logits:
\begin{equation}
    P(v|\mathbf{c}) = \frac{\exp(\mathbf{c}^\top\mathbf{e}_v)}{\sum_{u\in\mathcal{V}}\exp(\mathbf{c}^\top\mathbf{e}_u)},
\end{equation}
where $\mathcal{V}$ is the whole vocabulary.
While being a fundamental topic in natural language processing, previous work on interpreting them is usually focused at the word level, such as their semantic information~\citep{csenel2018semantic}, word senses~\citep{hewitt-etal-2023-backpack}, and analogical relations~\cite{mikolov2013efficient, park-etal-2017-rotated}. However, as the word embeddings are optimized for generation loss during pre-training, the learned embedding space should be closely associated with LMs' generation distributions. In this work, we propose to study the roles that word embeddings play in LM generation, which remains an underexplored topic, and analyze a simple while effective LM steering method \modelname{}.

We start by examining the theoretical relation between word embeddings and LM generation distributions.
We find that linear transformations in output word embeddings are equivalent to LM generation style changes.
This motivates our empirical investigation on a lightweight and simple steering method, \modelname{}, to steer LM generation distribution flexibly and transparently.
\modelname{} deploys a $d\times d$ learnable linear transformation $W$ on the output word embeddings, where $d$ is the embedding dimension.
Specifically, the embeddings $\mathbf{e}_v$ are replaced with $\mathbf{e}_v + \epsilon W \mathbf{e}_v$. Here, $\epsilon$ acts as a ``steering value'' to provide a simple control on steering polarity and intensity.
Inherently, \modelname{} discovers hidden dimensions in word embeddings that are associated with text styles. By transforming these dimensions, \modelname{} affects LMs' interpretation of words and, consequently, generation distributions.

Empirically, we find such \modelname{}s exist prevalently in LMs of all sizes, thanks to its general formulation and ignorance of model architectures. On tasks such as language model detoxification and sentiment control, it achieves comparable or superior performance compared to the state-of-the-art controlled generation baselines.
We also find multiple merits of \modelname{} both as a steering tool and lens for inspecting the relation between word embeddings and language model generation. 
An \modelname{} can \textbf{highlight} text spans that best indicate a style in the full text.
A learned \modelname{} also makes word embedding dimensions \textbf{interpretable} by pointing out the dimensions closest related to a style, revealing what kinds of words contribute to or contradict a style.
\modelname{} is both \textbf{parameter efficient} and \textbf{data-efficient}: on GPT2-large, it learns only parameters only 0.2\% the size of the original model (9\% of the size of LoRA~\citep{hu2021lora}, a parameter-efficient fine-tuning method) and able to text sentiments on dozens of sentences.
A learned \modelname{} is \textbf{transferable} to other LMs with different sizes by explicit-form calculation without additional training. 
Moreover, \modelname{} theoretically enables both \textbf{continuous} and \textbf{compositional} control. This allows for dealing with diverse and nuanced situations, such as fine-grained personalized or customized generation, without re-training for each scenario.

\section{Related Work}
\label{sec:related}

\noindent\textbf{Understanding Word Embeddings}
Language models learn word embeddings for individual words automatically after pre-training. To understand them, \cite{mikolov2013efficient, allen2019analogies} discover linear translational relations among embeddings, while \cite{park-etal-2017-rotated, rothe2016word, ethayarajh2019rotate} examine their rotational relations. Some other work examines the semantic information ~\citep{csenel2018semantic, lund1996producing, jang2017elucidating, csenel2022learning, murphy2012learning, faruqui2014retrofitting} or word senses \cite{panigrahi-etal-2019-word2sense,hewitt-etal-2023-backpack} of individual words in their embeddings. These efforts, however, have mostly focused on and evaluated word-level interpretations of word embeddings~\cite {chang2009reading} while we first investigate their relations with language model generations.

\noindent\textbf{Control of Language Models}
has been of growing interest in recent years, motivated by the increasing capabilities of LMs \cite{li-etal-2023-defining}.
This area originates from the need to leverage the generation capabilities of large language models while avoiding the need for time-consuming and costly retraining or fine-tuning.
Attempts include applying attribute classifiers or heuristic constraints at decoding time~\citep{kumar2022gradient, dathathriplug, liu2021dexperts, yang2021fudge}, treating the generation process as an optimization problem over the embedding or token sequences~\citep{kumar2021controlled}, or post-editing the output~\citep{li2018delete}. These techniques are often computationally expensive and rely on suitable external classifiers.
More recently, prompting-based control for large language models has received much attention, which, however, relies on the quality and availability of large language models~\citep{brown2020language, openai2023gpt4}, and may also necessitate the deliberate training~\citep{raffel2020exploring, zhou2023controlled}. It can also be challenging to design effective prompts for complex or nuanced control goals.
Parameter-efficient fine-tuning such as LoRA~\citep{hu2021lora} focuses on learning low-rank approximations of model parameters. However, this cannot achieve flexible and transferrable language model steering like ours.
Probably most closely related to our work are attempts to discover ``steering'' vectors or tokens~\citep{subramani2022extracting, li2021prefix}, and also similar work in image generation~\citep{jahaniansteerability,Hu2021}.
Different from our model, these efforts focus on other applications such as multi-task learning and sentence recovery, and the learned vectors are not shown to be transferrable or interpretable nor enable flexible control.

\noindent\textbf{Language Model Detoxification} 
Motivated by the goal to address the systematic biases embedded in language models, there are efforts in conducting language model de-biasing or de-toxification \citep{meade_empirical_2022,kaneko-etal-2022-debiasing}. Approaches span all aspects of the language model pipeline. A line of work focuses on automatically obtaining cleaner data \citep{barikeri-etal-2021-redditbias, webster2021measuring, dinan-etal-2020-queens}. Another line of work modifies the model workflow design to explicitly accommodate the bias factors \citep{webster2021measuring, 10.1162/tacl_a_00434, debiasinggradient2023,debiasing2023b, Adept2023}. The most related line of work to the herein proposed method involves manipulating embedding space such as Principle Component Analysis and Nullspace Projection \citep{liang-etal-2020-towards, bolukbasi2016man, ravfogel-etal-2020-null}. The evaluation in these settings \citep{kaneko-bollegala-2021-debiasing, nadeem-etal-2021-stereoset, nangia-etal-2020-crows} mostly consists of quiz-question checking for stereotypical misbeliefs. More related to our method are those mentioned in language model control~\citep{kumar2022gradient, dathathriplug, liu2021dexperts, yang2021fudge, kumar2021controlled}, which constrains or guides text generation according to a classifier.
A unique contribution in our work is that the learned \modelname{} can be transferred to detoxify other off-the-shelf language models without a costly training process.

\section{\modelname{}: Revealing Hidden Steers in Word Embeddings}
\label{sec:method}

As a theoretical motivation, we first show an informal theorem relating output word embeddings with generation styles.
We leave the formal statement as well as the proof to Appendix~\ref{appsec:proof} and only present an intuitive interpretation.

\begin{theorem}
\label{thm:informal}
(Informal) With certain assumptions, shifting styles in language models is equivalent to a linear transformation in word embedding space.
\end{theorem}

Inspired by this discovery, we propose \modelname{} to apply a linear transform in the output word embedding space. \modelname{} is conceptually simple and straightforward to implement.
An illustration of \modelname{} is presented in Figure~\ref{fig:overview}(a).
Specifically, we assume a language model with fixed parameters. We replace its each output word embeddings $\mathbf{e}_v$ with 
\begin{equation}
\mathbf{e}_v'=\mathbf{e}_v+\epsilon W\mathbf{e}_v=(I+\epsilon W)\mathbf{e}_v,
\end{equation}
and call the resulting language model $P_{\epsilon W}$ a ``\modelname{}ed model''. Here, the ``steer matrix'' $W$ is the only learnable parameter in \modelname{}, and $\epsilon$ is an adjustable scalar indicating the polarity and intensity of the ``steering value''. Without loss of generality, we arbitrarily pick a small value $\epsilon_0=1e-3$ as the default steering value.\footnote{Using an arbitrary $\epsilon_0$ possess the same representation power as any other $\epsilon$, as there always exists $W'=\frac{\epsilon}{\epsilon_0}W$ so that $\epsilon_0 W' = \epsilon W$ holds.} We use $P_{\epsilon W}$ to denote the steered language model.
Figure~\ref{fig:overview}(b, c) shows the training and generation process of \modelname{}. During training, we use the positively steered model $P_{\epsilon W}$ to fit the positively labeled texts, with maximal likelihood as the training objective. When negative texts are available, we also fit them with $P_{-\epsilon W}$. When generating with \modelname{}, the user only needs to specify a steering value $\epsilon$ and then decode the language model.
More details are in Appendix~\ref{appsec:implementation}.
Intuitively explaining, \modelname{} matrix $W$ learns to identify word embedding dimensions that are best associated with a target style and manipulate among those dimensions to achieve language rewording. In Section~\ref{subsec:interpretability}, we show this enables an interpretation of word embeddings by analyzing the learned $W$.

We also theoretically and empirically compare \modelname{} against a simplified version, a \textbf{(soft) word blacklist (SWB)}: learning a global logit offset is applied to each token candidate after the original logits are computed. As we demonstrate in Appendix~\ref{appsec:blacklist}, adding a (learnable) vector to context vectors $\boldsymbol{c}$ achieves a similar effect with SWB. We also further prove that \modelname{} is theoretically expressive of any distribution shift, while a SWB is unable to do so. In \ref{sec:applications}, we show that SWB indeed yields inferior performance than \modelname{}.

\section{Steering Language Model Generation}
\label{sec:applications}

\begin{table*}[t!]
\centering
\linespread{1}

\captionsetup{aboveskip=3pt}\captionsetup{belowskip=0pt}
\setlength{\tabcolsep}{1.5mm}{
\begin{tabular}{lccccccc}
\toprule

\multirow{2}{*}{\textbf{Model}} & \textbf{Backbone} & \multicolumn{2}{c}{\textbf{Toxicity}$\downarrow$} & \textbf{Fluency} & \multicolumn{2}{c}{\textbf{Diversity}$\uparrow$} \\
 & \textbf{Size} & Max. toxicity & Toxicity prob. & Output ppl.$\downarrow$ & Dist-1 & Dist-2 & Dist-3 \\

 \midrule
GPT-2 (original)
& 117M & 0.527 & 0.520 & 25.45 & 0.58 & 0.85 & 0.85 \\

\midrule

PPLM (10\%) & 345M
& 0.520 & 0.518 & 32.58 & 0.58 & 0.86 & 0.86 \\

DAPT & 117M
& 0.428 & 0.360 & 31.21 & 0.57 & 0.84 & 0.84 \\

GeDi & 1.5B
& 0.363 & 0.217 & 60.03 & 0.62 & 0.84 & 0.83 \\

DExperts$_{\text{base}}$ & 117M
& 0.302 & 0.118 & 38.20 & 0.56 & 0.82 & 0.83 \\

DExperts$_{\text{medium}}$ & 345M
& 0.307 & 0.125 & 32.51 & 0.57 & 0.84 & 0.84 \\

DExperts$_{\text{large}}$ & 762M
& 0.314 & 0.128 & 32.41 & 0.58 & 0.84 & 0.84 \\

PromptT5 & 780M
& 0.320 & 0.172 & 55.1 & 0.58 & 0.76 & 0.70 \\

MuCoLa & 762M
& 0.308 & 0.088 & 29.92 & 0.55 & 0.82 & 0.83 \\

LoRA & 762M
& 0.365 & 0.210 & 21.11 & 0.53 & 0.85 & 0.86 \\

Soft-Blacklist & 762M
& 0.270 & 0.154 & 18.28 & 0.53 & 0.81 & 0.83 \\

\midrule

\modelname{}$_{\text{base}}$ & 117M
& 0.296$_{\pm 0.018}$ & 0.129$_{\pm 0.012}$ 
 & 36.87  & 0.54 & 0.86 & 0.86 \\

\modelname{}$_{\text{medium}}$ & 345M
& \textbf{0.215}$_{\pm 0.015}$ & \textbf{0.059}$_{\pm 0.029}$
 & 43.56 & 0.56 & 0.83 & 0.84 \\

\modelname{}$_{\text{large}}$ & 762M
& 0.249$_{\pm 0.007}$ & 0.089$_{\pm 0.009}$ & 28.26 & 0.55 & 0.84 & 0.84 \\

\bottomrule

\end{tabular}

}
\caption{On language model detoxification, \modelname{} achieves best performance. $\pm$ denotes standard deviation on 3 random seeds.}
\label{tab:detoxification}
\end{table*}

\begin{figure}[h!]
\centering
\includegraphics[width=0.5\textwidth]{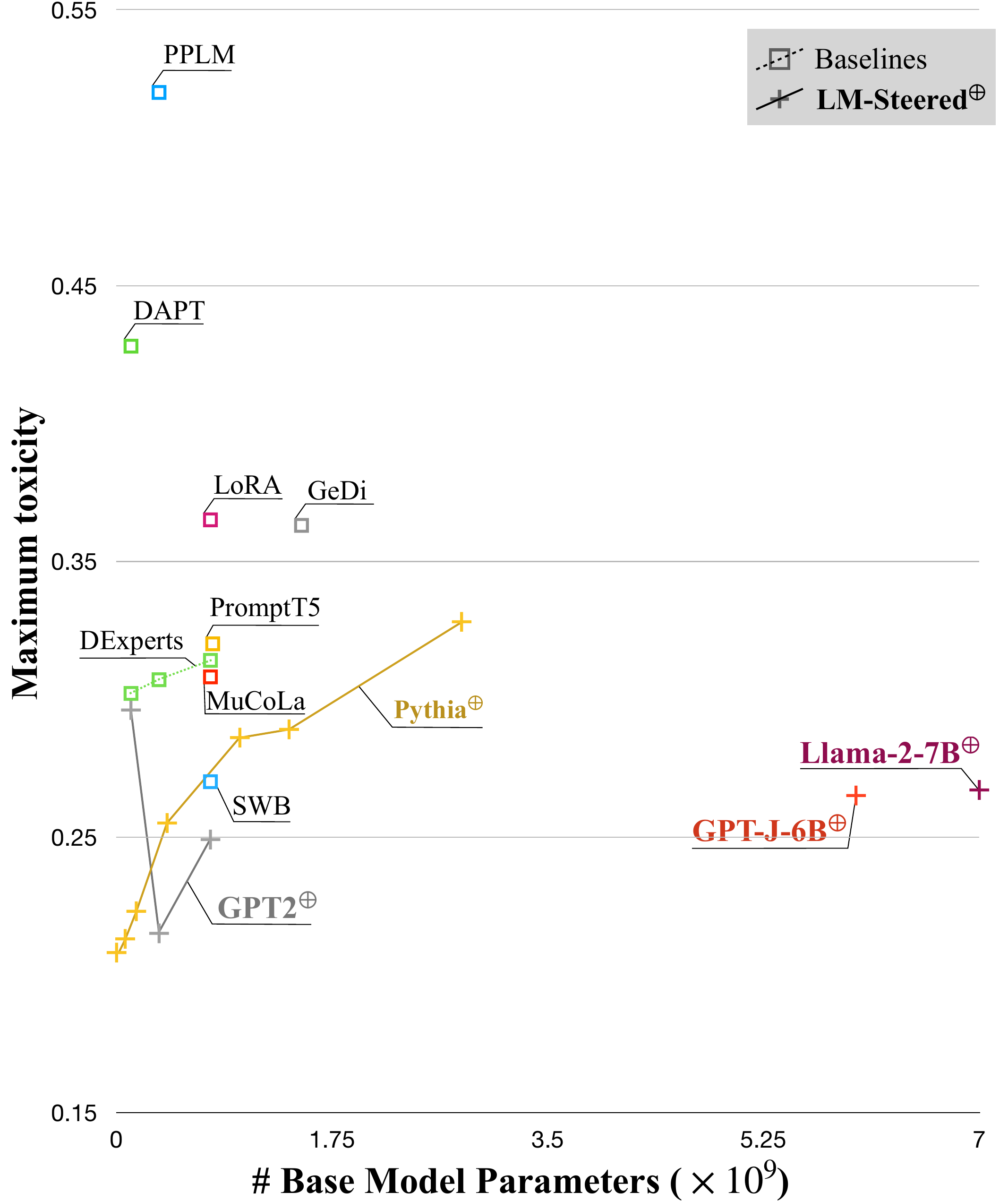}
\caption{Across base model sizes, \modelname{}ed GPT2 family, Pythia family, GPT-J and Llama-2-7B models (+) consistently outperform other baselines (($\square$)) on detoxification. X$^\oplus$ means an \modelname{}ed language model X.}
\vspace{-4mm}
\label{fig:detoxification}
\end{figure}
\begin{table*}[t!]
\centering
\setlength{\tabcolsep}{1.1mm}{
\begin{tabular}{lccc|ccc|ccc}
\toprule
& \textbf{\modelname{}} & \textbf{Tie} & \textbf{LoRA} & \textbf{\modelname{}} & \textbf{Tie} & \textbf{GPT-2} & \textbf{\modelname{}} & \textbf{Tie} & \textbf{DExperts} \\
\midrule
\textbf{Detoxified} & \textbf{19.0} & 69.5 & 11.5 & \textbf{24.5} & 56.5 & 19.0 & \textbf{24.0} & 56.5 & 19.5 \\
\textbf{Fluent} & \textbf{21.0} & 69.0 & 10.0 & 21.0 & 57.5 & \textbf{21.5} & \textbf{25.0} & 52.0 & 23.0 \\
\textbf{Topical} & \textbf{18.0} & 69.5 & 12.5 & \textbf{32.0} & 47.0 & 21.0 & \textbf{32.0} & 56.5 & 11.5 \\
\bottomrule
\end{tabular}
}
\caption{Human evaluation results by comparing with LoRA, GPT-2 and DExperts. \modelname{} wins out on most metrics while being comparable to GPT-2 on fluency.}
\label{tab:human_eval}
\end{table*}

\subsection{Language Detoxification}
\label{subsec:detoxification}

It is known that large pre-trained LMs might generate toxic content that appears in the pre-training distribution~\cite{sheng2019woman, gehman2020realtoxicityprompts}, such as inaccurate information, harmful stereotypes, and unethical content.
Language model detoxification is the task of mitigating or avoiding these generations in order to enable the safe usage of language models. We experiment on GPT2 family~\citep{radford2019language}, Pythia family~\citep{biderman2023pythia}, GPT-J~\citep{gpt-j} and Llama-2~\citep{touvron2023llama} as the backbone language models.

\textbf{Setting:}
Following~\cite{liu2021dexperts}, we use Jigsaw Unintended Bias in Toxicity Classification Kaggle challenge\footnote{\url{https://bit.ly/3cvG5py}} as the training dataset.
For evaluation, we use 10K nontoxic prompts from the RealToxicityPrompts dataset~\citep{gehman2020realtoxicityprompts}. We randomly generate 25 sentences of up to 20 tokens using nucleus sampling~\citep{holtzmancurious} with $p=0.9$. Then the toxicity scores (in range $[0, 1]$) of generations are evaluated using Perspective API~\footnote{\url{https://perspectiveapi.com}}. Two metrics are reported: the maximal toxicity of generations on each prompt averaged across prompts (``Avg. max. toxicity'') and the averaged probability of generating $>0.5$ toxicity (``Toxicity prob.''). We also evaluate generation quality in terms of fluency (perplexity score measured by a GPT2-large) and diversity (Dist-\{1, 2, 3\}: the portion of distinct \{1, 2, 3\}-grams). When decoding, we use a steering value of $5\epsilon_0$ for generation, selected based on the balance between generation fluency and task performance on the dev set in Appendix~\ref{appsec:ablation}.

\textbf{Baselines:}
\textbf{DExperts}~\citep{liu2021dexperts} trains positive and negative label classifiers and uses the difference in two classifiers' scores to offset the LM's original logits.
\textbf{DAPT}~\citep{gururangan2020don} simply further pretrains the language model on the non-toxic subset (filtered by Perspective API) of OpenWebText Corpus (OWT)~\citep{gokaslan2019openwebtext}.
\textbf{PPLM}~\citep{dathathriplug} learns to use the gradients of the label classifier to update the LM's hidden representations.
\textbf{GeDi}~\citep{krause2021gedi} is a model that uses the Bayesian rule for class-conditioned LM generation.
\textbf{MuCoLa}~\citep{kumar2022gradient} models the text generation as an optimization problem regarding the classifier scores.
\textbf{PromptT5}~\citep{raffel2020exploring} T5 is a pre-trained LM optimized for prompt-based task solving, and we use ``Complete this sentence so that it embodies a \{positive/negative\} sentiment:'' to prompt T5. 
\textbf{LoRA}~\citep{hu2021lora} trains low-rank approximations of parameter matrices to achieve parameter-efficient fine-tuning. Finally, we compare with the soft blacklist baseline discussed in Section~\ref{sec:method}.

\textbf{Results and Analysis:}
We present the results in Table~\ref{tab:detoxification}. Despite the simple design, \modelname{} achieves the best detoxification scores on both metrics, reducing Avg. max. toxicity by $>6\%$ absolute percentages. It is also noteworthy that \modelname{} also demonstrates reasonable balance on fluency (2nd lowest perplexity score) and diversity (same-level Dist-k scores with baselines). Figure~\ref{fig:detoxification} further shows the detoxification versus baseline size, where \modelname{}+\{GPT2 family, Pythia family, GPT-J and Llama-2\} uniformly outperforms baselines where of all sizes, where more numerical results can be found in Appendix~\ref{appsec:pythia} and ~\ref{sec:gpt-j}. Incorporation of \modelname{} with LoRA, instruction following, and full embedding tuning are explored in Appendix~\ref{appsec:plus_lora}, \ref{appsec:plus_instruction} and \ref{appsec:plus_tuning}, respectively.

\textbf{Human Evaluation} We compare LM-steer with LoRA, DExperts, and GPT-2 in a pairwise manner with human annotators. Specifically, we follow the practice in DExperts and ask four student human annotators to compare 50 generations from LM-steer and the baseline from 3 perspectives: detoxification, fluency, and being topical to the prompt. The results are as follows. We can see that LM-steer is ranked significantly less toxic and more topical than the baseline. It performs similarly to DExperts and GPT-2 but better than LoRA in terms of fluency.

\begin{table*}[t!]
\centering
\linespread{1}


\setlength{\tabcolsep}{1.4mm}{
\begin{tabular}{lcccccccc}
\toprule

\multirow{3}{*}{\textbf{Target}} & \multirow{3}{*}{\textbf{Model}} & \multicolumn{3}{c}{\textbf{Sentiment Positivity} / \%} & \textbf{Fluency} & \multicolumn{2}{c}{\textbf{Diversity}$\uparrow$}  \\

& & Positive & Neutral & Negative & \multirow{2}{*}{Output ppl.$\downarrow$} & \multirow{2}{*}{Dist-1} & \multirow{2}{*}{Dist-2} & \multirow{2}{*}{Dist-3} \\
& & prompts & prompts & prompts \\

\midrule

\multirow{12}{*}{\textbf{Positive}$\uparrow$} 

& \modelname{}$_{\text{large}}$
& & 90.70 & 41.23 & 41.20 & 0.46 & 0.78 & 0.83 \\

& \modelname{}$_{\text{medium}}$
& & \textbf{95.36} & 56.98 & 67.68 & 0.46 & 0.77 & 0.80 \\

& \modelname{}$_{\text{base}}$
& & 90.46 & \textbf{57.26} & 54.38 & 0.47 & 0.78 & 0.81 \\

\cmidrule(lr){2-9}

& Soft-Blacklist
& & 86.40 & 25.64 & 99.46 & 0.42 & 0.76 & 0.81 \\

& LoRA
& & 26.88 & 7.20 & 158.56 & 0.57 & 0.82 & 0.83 \\

& DExperts$_{\text{large}}$
& & 94.46 & 36.42 & 45.83 & 0.56 & 0.83 & 0.83 \\

& DExperts$_{\text{medium}}$
& & 94.31 & 33.20 & 43.19 & 0.56 & 0.83 & 0.83 \\

& DExperts$_{\text{small}}$
& & 94.57 & 31.64 & 42.08 & 0.56 & 0.83 & 0.84 \\

& DExperts (pos)
& & 79.83 & 43.80 & 64.32 & 0.59 & 0.86 & 0.85 \\

& GeDi
& & 86.01 & 26.80 & 58.41 & 0.57 & 0.80 & 0.79 \\

& DAPT
& & 77.24 & 14.17 & 30.52 & 0.56 & 0.83 & 0.84 \\

& PPLM (10\%)
& & 52.68 & 8.72 & 142.11 & 0.62 & 0.86 & 0.85 \\

& PromptT5
& & 68.12 & 15.41 & 37.3 & 0.58 & 0.78 & 0.72 \\

\midrule
& GPT-2 (original)
& 99.08 & 50.02 & 0.00 & 29.28 & 0.58 & 0.84 & 0.84 \\

\midrule

\multirow{12}{*}{\textbf{Negative}$\downarrow$} 

& PromptT5
& 69.93 & 25.78 & & 48.6 & 0.60 & 0.78 & 0.70 \\

& PPLM (10\%)
& 89.74 & 39.05 & & 181.78 & 0.63 & 0.87 & 0.86 \\

& DAPT
& 87.43 & 33.28 & & 32.86 & 0.58 & 0.85 & 0.84 \\

& GeDi
& 39.57 &  8.73 & & 84.11 & 0.63 & 0.84 & 0.82 \\

& DExperts (neg)
& 61.67 & 24.32 & & 65.11 & 0.60 & 0.86 & 0.85 \\

& DExperts$_{\text{small}}$
& 45.25 & 3.85 & & 39.92 & 0.59 & 0.85 & 0.84 \\

& DExperts$_{\text{medium}}$
& 40.21 & 3.79 & & 43.47 & 0.59 & 0.85 & 0.84 \\

& DExperts$_{\text{large}}$
& \textbf{35.99} & \textbf{3.77} & & 45.91 & 0.60 & 0.84 & 0.83 \\

& LoRA
& 57.71 & 20.08 & & 192.13 & 0.55 & 0.78 & 0.79 \\

& Soft-Blacklist
& 73.72 & 14.28 & & 50.95 & 0.38 & 0.70 & 0.76 \\

\cmidrule(lr){2-9}

& \modelname{}$_{\text{base}}$
& 57.26 & 10.12 & & 51.37 & 0.49 & 0.77 & 0.79 \\

& \modelname{}$_{\text{medium}}$
& 52.32 & 7.10 & & 71.48 & 0.47 & 0.77 & 0.79 \\

& \modelname{}$_{\text{large}}$
& 54.84 & 8.02 & & 57.74 & 0.48 & 0.78 & 0.80 \\

\bottomrule
\end{tabular}
}

\captionsetup{aboveskip=3pt}\captionsetup{belowskip=0pt}
\caption{Results on sentiment control task.
The upper half displays a positive control task and requires a higher positivity score and vice versa for the lower half. \modelname{} gets the best metrics on the positive side and 2nd to 3rd places on the negative side despite being simpler and smaller. 
For backbone model sizes, please refer to Table~\ref{tab:detoxification}.
}

\label{tab:sentiment}
\end{table*}

\begin{figure*}[t!]
    \centering
    \begin{subfigure}[b]{0.45\textwidth}
        \centering
        \includegraphics[width=0.9\textwidth]{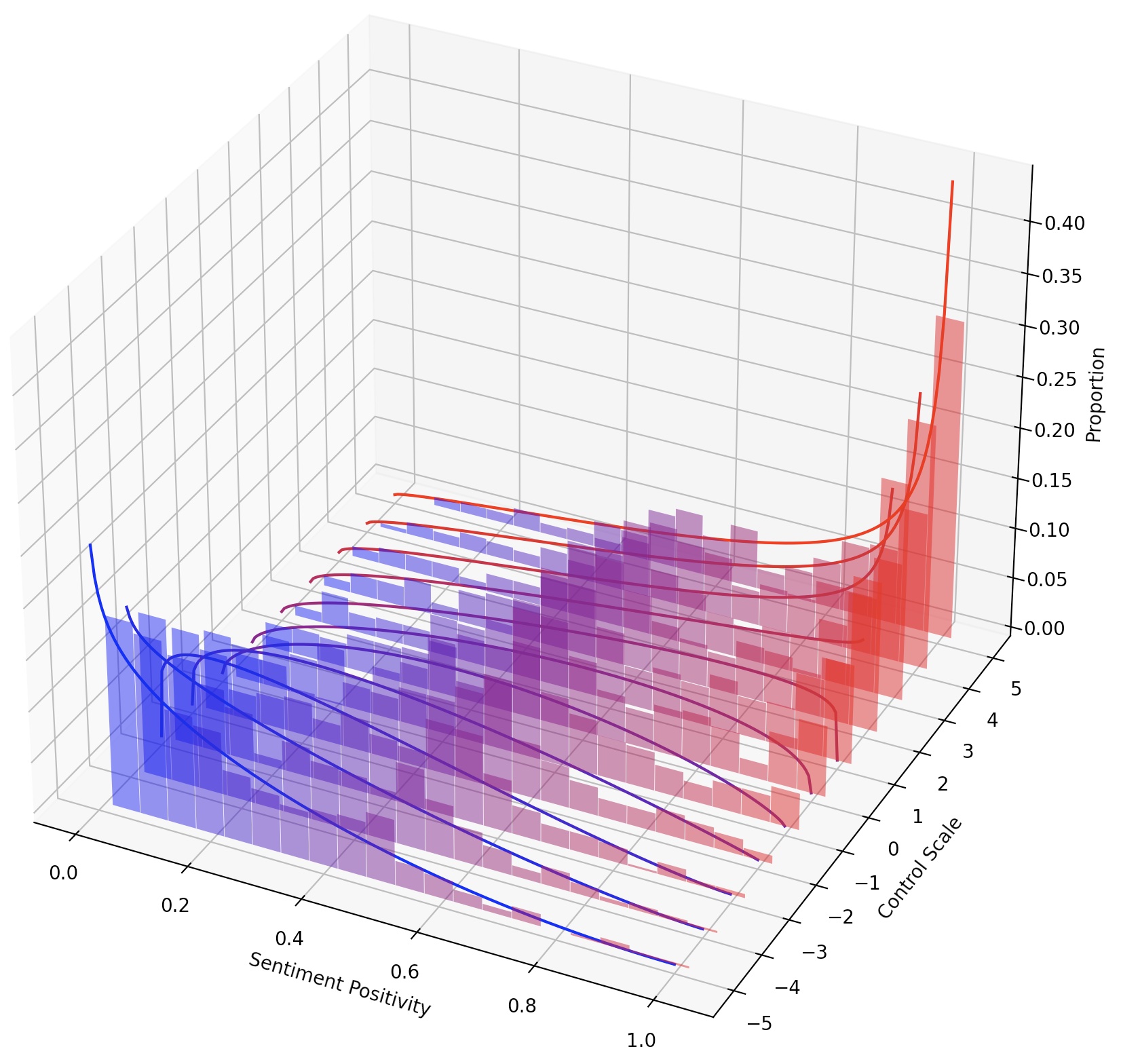}
        \caption{Continuous control on sentiment with $\epsilon$ in $-5\epsilon_0\sim 5\epsilon_0$ results in a sentiment distribution shift. Color indicates sentiment and height indicates frequency/density.}
        \label{fig:linear}
    \end{subfigure}
    \hfill
    \begin{subfigure}[b]{0.45\textwidth}
        \centering
        \includegraphics[width=\textwidth]{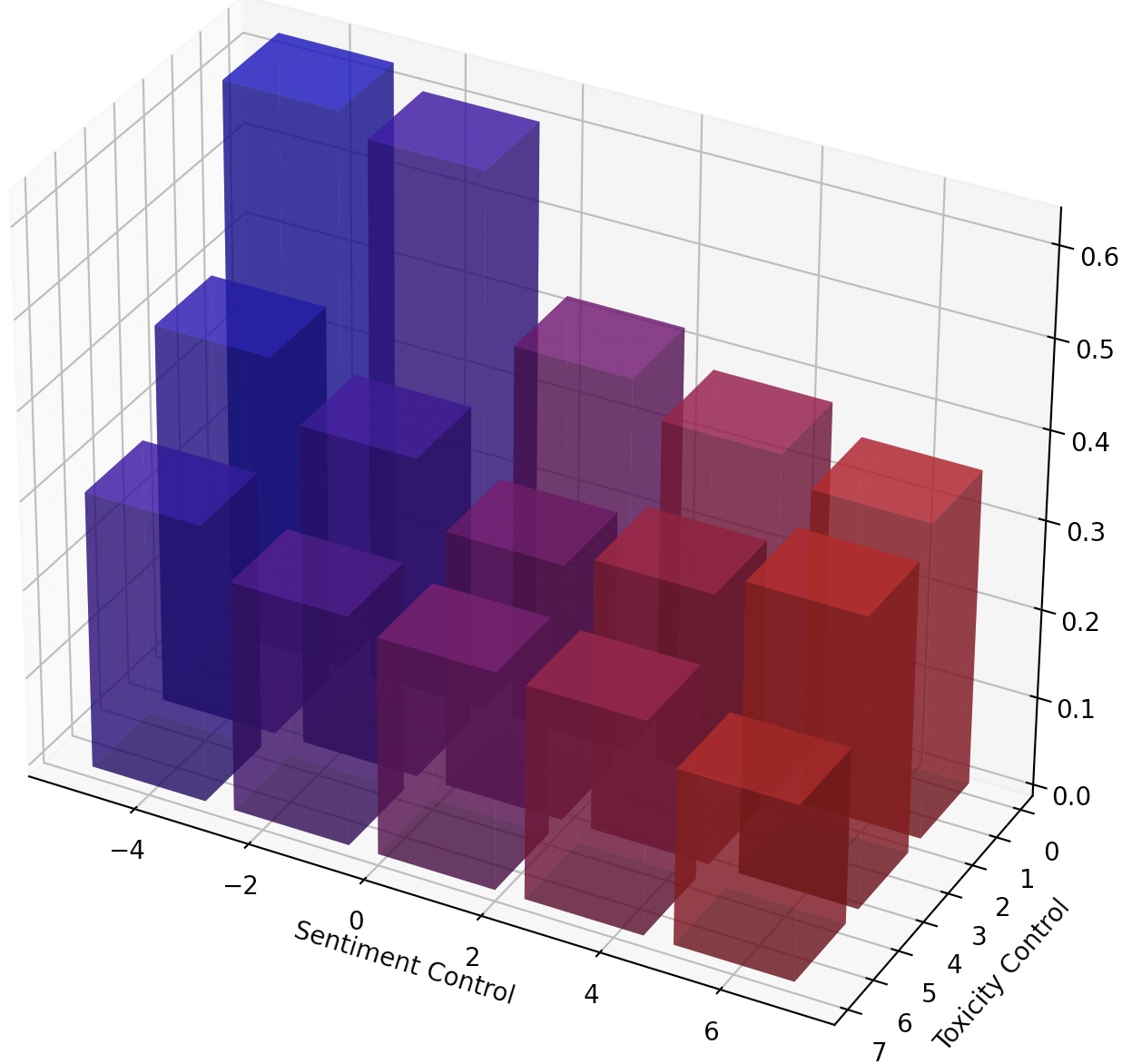}
        \caption{Compositional control sentiment ranging in $-5\epsilon_0\sim 5\epsilon_0$ and toxicity in $0\sim 5\epsilon_0$. Color means sentiment and height is toxicity.}
        \label{fig:compositional}
    \end{subfigure}
    \caption{Continuous and compositional control using \modelname{}.}
    \label{fig:linear_compositional}
\end{figure*}

\begin{table*}[t]
\centering
\linespread{1}

\vspace{-1mm}
\begin{tabular}{lcccccccc}
\toprule

& \textbf{\modelname{}} & \textbf{DAPT} & \textbf{GeDi} & \textbf{CTRL} & \textbf{PPLM} & \textbf{DExpert} & \textbf{MuCoLa} & \textbf{LoRA} \\

\midrule

\textbf{Parameters} & \textbf{1.6M} & 355M & 355M & 355M & 124M & 355M & 898M & 18M \\
\textbf{Speed Ratio} & 1.24 & \textbf{1.00} & 2.94 & 3.79 & 270.11 & 1.98 & 24.03 & \textbf{1.00}  \\

\bottomrule
\end{tabular}

\captionsetup{aboveskip=3pt}\captionsetup{belowskip=0pt}
\caption{Decoding time and learnable parameter efficiency. Time efficiency is measured by relative decoding time compared to the base language model. The best numbers are bolded.}

\label{tab:efficiency}
\end{table*}

\subsection{Sentiment Control}
\label{subsec:sentiment}

We also evaluate \modelname{}'s performance on an extensively studied generation task controlled by sentiment. This ability can be found useful when tailoring persuasive and emotionally appealing messages to specific target audiences in marketing or advertising or to create personalized and engaging user experiences in chatbot systems.

\textbf{Setting:}
We follow the setting in \cite{liu2021dexperts} and use Stanford Sentiment Treebank (SST-5)~\citep{socher2013recursive} as training data, where we use texts with labels 1$\sim$2 as negative samples, and those with 4$\sim$5 labels as positive samples. 
For evaluation, we use the HuggingFace’s sentiment classifier~\citep{wolf2020transformers}. The generation prompts are a subset of the OpenWebText Corpus filtered by the sentiment analysis classifier. Models are applied on these prompts 25 times to generate up to 20 tokens. We then measure the average percentage of positive generations for each prompt as the ``Positivity'' score. Similar to the detoxification task, we use $5\epsilon_0$ for positive sentiment and $-5\epsilon_0$ for negative sentiment control.

\textbf{Baselines:}
In addition to the baselines used in detoxification, two variants of DExperts, DExperts (pos) and DExperts (neg), which only use one of the two classifiers for guiding generation, are also listed.

\textbf{Results:}
Table~\ref{tab:sentiment} presents the full results. \modelname{}, despite a much simpler and smaller model, takes 1st place on the positive side and 2nd or 3rd place on the negative side and achieves a reasonable balance on fluency and diversity.
\begin{table}[h!]
\centering
\linespread{1}

\begin{tabular}{cp{6cm}}

\toprule

\textbf{Steer} & \textbf{Generation} \\

\midrule

-5e-3 & What \textbf{moron} said that \textbf{stupid} comment. \\
\midrule
-3e-3 & What's \textbf{stupid} is \textbf{stupid}, right? \\
\midrule
-1e-3 & What's this? You think that your religion, your culture, your country are \textbf{not good enough}? \\
\midrule
0 & What's more, it makes for a fun, cheap, and efficient way to improve the performance of your car engine and to make your driving that much safer. \\
\midrule
1e-3 & What's more, it makes for a fun, cheap, and efficient way to improve the performance of your car engine and motor. \\
\midrule
3e-3 & What's on your mind?  What's on your mind? \\
\midrule
5e-3 & What's on Netflix? If you can't figure out what's being watched on Netflix, you need to figure out what are people watching! \\

\bottomrule
\end{tabular}

\captionsetup{aboveskip=3pt}\captionsetup{belowskip=0pt}
\caption{\modelname{} continuously steers GPT2-Large generation from toxic to non-toxic when interpolating and extrapolating the steering value. Both the number and intensity of \textbf{toxic words} decrease with increased steering value.}
\label{tab:gradual}
\end{table}
\begin{table*}[t!]
\centering
\linespread{1}

\vspace{-1mm}
\begin{tabular}{cp{14cm}}

\toprule

\textbf{Dim.} & \textbf{Matched Words} \\

\midrule

0 & mor, bigot, Stupid, retarded, coward, stupid, loser, clown, dumb, Dumb, losers, stupidity, garbage
, idiots, fools, idiot, lame
\\
\midrule
1 & stupid, idiot, Stupid, idiots, jerk, pathetic, suck, buff, stupidity, mor, damn, ignorant, fools, dumb
, disgusting , damned, narcissistic, troll
\\
\midrule
3 & idiot, godd, damn, \\
\midrule
5 & Balk, lur, looms, hides, shadows,
Whites, slippery, winds\\
\midrule
7 & bullshit, fiat,  shit,  lies, injust, manipulation \\
\midrule
8 & disabled, inactive, whip, emo, partisan, spew, bombed, disconnected, gun, failing, Republicans
, defeated, Jeb, blowing , bombard, ineffective, reload, destructive, flo, blown
\\
\midrule
9 & winners, upside \\

\bottomrule
\end{tabular}

\captionsetup{aboveskip=3pt}\captionsetup{belowskip=0pt}
\caption{Word embedding dimensions that are most influenced by \modelname{} on detoxification task.}
\label{tab:interpretability}
\end{table*}

\subsection{Continuous and Compositional Control}
\label{subsec:linearity}
The conceptually simple design of \modelname{} makes it an architecture-agnostic plug-in to diverse language models. We demonstrate that \modelname{} maintains a linearity guarantee, which enables continuous and compositional control. More specifically, our model allows for interpolation and extrapolation on the steering spectrum by simply interpolating and extrapolating the steering value. Moreover, if two \modelname{}s $\epsilon_1 W_1, \epsilon_2 W_2$ are learned on potentially different tasks, their effect can be combined by decoding with $P_{\epsilon1 W_1 + \epsilon_2 W_2}$.

In Figure~\ref{fig:linear}, we plot the distribution shift when adjusting sentiment steer $\epsilon$. We also curve the maximal likelihood estimated Beta distribution. In Figure~\ref{fig:compositional}, we observe that \modelname{} can compositionally control sentiment and toxicity, even though there exists a mutual influence between these two factors (e.g., a negative sentiment might also lead to more toxic comments). Table~\ref{tab:gradual} also provides an example of how the generated sentence is continuously steered from toxic to non-toxic, demonstrating a simple fine-grained control on the toxicity level. When the steering value increases from negative to positive, both the number and the intensity of toxic words (bolded in the table) decrease.

\begin{figure*}[t]
    \centering
    \begin{subfigure}[b]{0.48\textwidth}
        \centering
        \includegraphics[width=\textwidth]{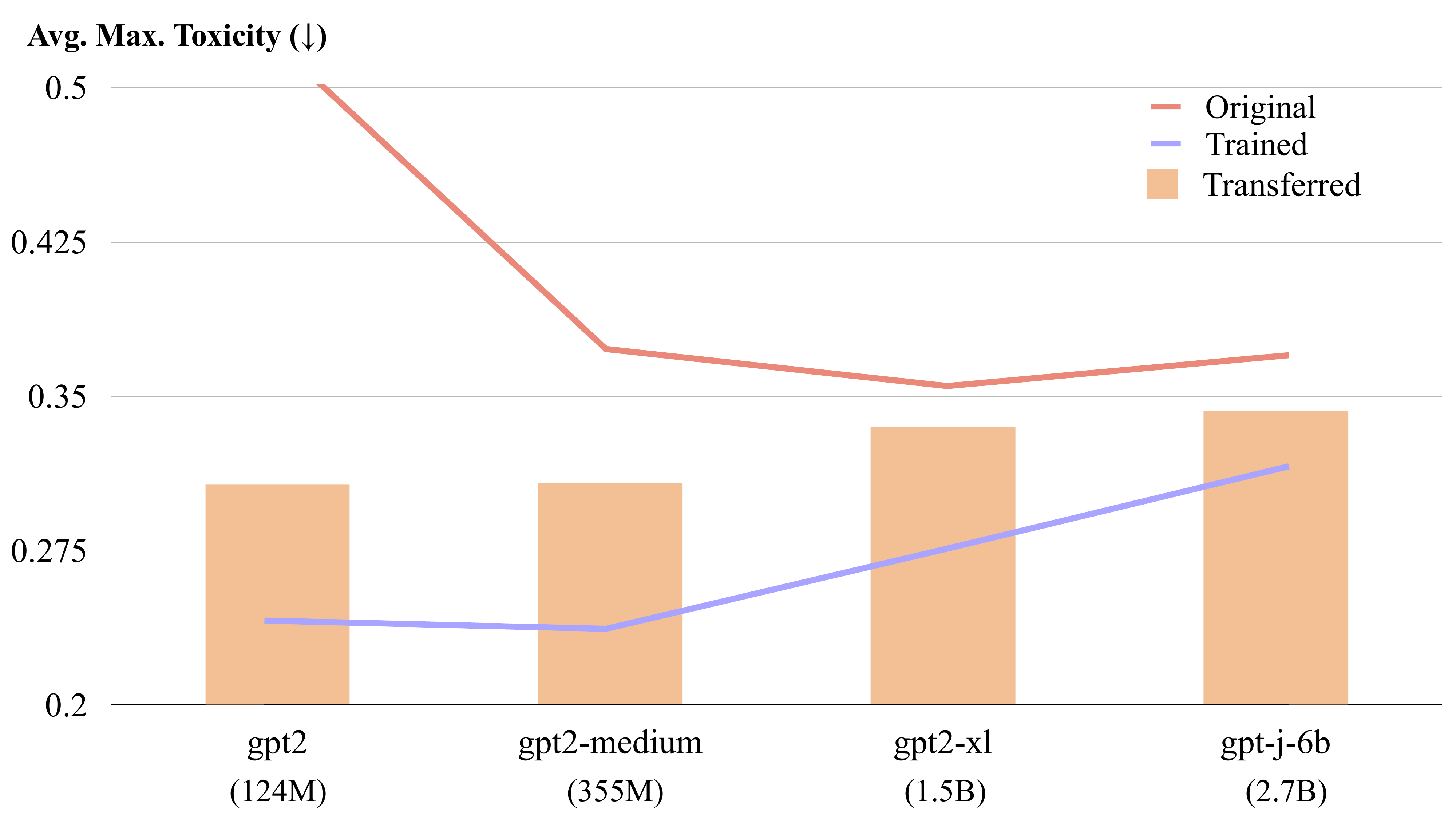}
        \caption{Transferring an \modelname{} to other LMs with explicit-form calculation. The transferred \modelname{} maintains the detoxification ability partially.}
        \label{fig:transfer}
    \end{subfigure}
    \hfill
    \begin{subfigure}[b]{0.48\textwidth}
        \centering
        \includegraphics[width=\textwidth]{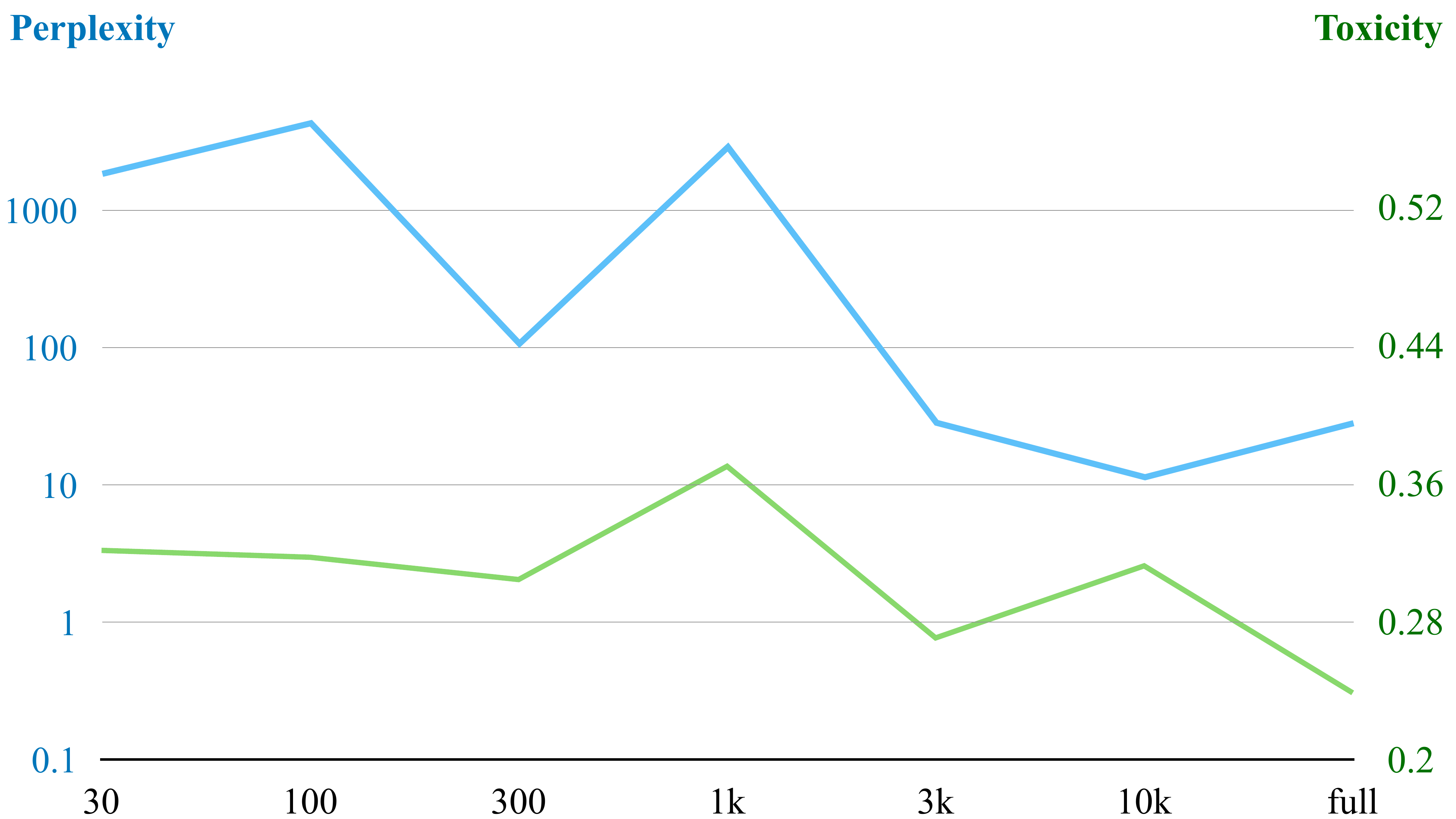}
        \caption{\modelname{} is capable of learning from only dozens of data samples to achieve decent detoxification. More data are beneficial for fluency.}
        \label{fig:efficiency}
    \end{subfigure}
    \caption{Measuring the transferability and data efficiency of \modelname{}.}
    \label{fig:transfer_efficiency}
\end{figure*}

\subsection{Efficiency}
\label{subsec:efficiency}
\label{sec:efficiency}
Thanks to its simple design, \modelname{} enjoys efficiency in multiple perspectives. We vary the detoxification dataset size from 30 to 10k and measure \modelname{}'s performance in Figure~\ref{fig:transfer_efficiency}(b). We see that as few as 30 data points still enable \modelname{} to achieve high detoxification scores (0.322). When dataset size exceeds 3k \modelname{} acquires a good balance between detoxification and generation quality. We also show decoding time and parameter efficiency in Table~\ref{tab:efficiency}, where our model only uses 1\% of the baseline's learnable parameter size and uses a low computation overhead during decoding.

\begin{table}[ht]
\centering
\linespread{1}

\vspace{-1mm}
\begin{tabular}{p{7.5cm}}

\toprule

There's another controversial \colorbox{yellow}{Hollywood racial} decision that Stacey Dash is sinking her teeth into. \\
\midrule

The UFC champ then suggested Justino is a longtime PED user with her \colorbox{yellow}{most d**ning} comments.\\
\midrule

But I really have a question for you: Why would I go on a game show and play into the \colorbox{yellow}{bulls**t} allowing myself to be ranked by some fake competition?\\
\midrule

I \colorbox{yellow}{think sexism} prevents this from being a real win for fat people. \\

\midrule

If they want to be fair and non \colorbox{yellow}{hypocritical idiots they} should.\\

\bottomrule
\end{tabular}

\captionsetup{aboveskip=3pt}\captionsetup{belowskip=0pt}
\caption{Toxic sentences with \colorbox{yellow}{toxic keywords} highlighted by \modelname{} after training detoxification on GPT2-Large.}
\vspace{-4mm}
\label{tab:evidence}
\end{table}
\section{\modelname{}s Connect Word Embeddings with the Text Distribution}
\label{sec:analysis}
In previous sections, \modelname{} revealed the hidden biases encoded in the automatically learned word embeddings of LMs. This section provides an alternative perspective, where these hidden biases serve as a lens for interpreting the connection between word embeddings and the generation distribution of LMs. Section~\ref{subsec:interpretability} demonstrates how \modelname{} decomposes word embeddings into interpretable dimensions closely related to word selection in a particular distribution. Section~\ref{subsec:evidence} shows that \modelname{} can highlight indicative words. Finally, in light of the hypothesis that different LMs have correlated word embedding spaces, Section~\ref{subsec:transfer} illustrates how \modelname{} can be transferred between LMs with an explicit expression.

\subsection{Interpreting Word Embeddings}
\label{subsec:interpretability}
\modelname{} provides a lens on how word embeddings correlate with LM word embeddings: \textit{what word dimensions contribute to or contrast to a specific style.}
In the detoxification experiment, we conduct an SVD decomposition of the learned $W$. Among $S, V, D$, the $D$ component can be interpreted as a ranked list of the most ``magnified'' row dimensions in the transformation $W$. We then take its first 9 rows and list the most influenced words in Table~\ref{tab:interpretability}. Dimensions 2, 4, and 6 are filtered out as they only match non-English tokens. Although offensive to read, this table helps us understand what kind of words are most related to toxicity and thus suppressed by \modelname{} in a generation.
More details are explained in Appendix~\ref{appsec:interpretability_details}.

\subsection{Highlighting Keywords in Styled Texts}
\label{subsec:evidence}
\modelname{} also enables automatically pointing out \textit{what specific words are most indicative of the style in a given sentence}. We conjecture that these distinctive words contribute most to the likelihood difference between a task-specific LM and a domain-general LM. In Table~\ref{tab:evidence}, we list a few toxic sentences and the most indicative text spans highlighted by \modelname{}. To acquire such text spans, we calculate $\log P_{\epsilon W}(v_i | \cdots v_{i-1}) - \log P_0(v_i | \cdots v_{i-1})$ which is the difference of log-likelihoods of each token under steered ($P_{\epsilon W}$) and the original language model ($P_0$). Then, we deploy a dynamic programming algorithm to obtain the continuous sub-sequence under length 5 with the highest cumulative difference. We list highlighted words in toxic prompts in the RealToxicityPrompts dataset in Table~\ref{tab:evidence}, which shows insulting, cursing, controversial, and sexually explicit words in each sentence.

\subsection{Transfering \modelname{} Between Models}
\label{subsec:transfer}
A much-desired property of \modelname{}, because of its theoretical soundness, is its transferability to other language models. 
Details and derivations of \modelname{} transfer are in Appendix~\ref{appsec:transfer_details}. Intuitively speaking, the original logit $\mathbf{c}^\top \mathbf{e}_v$ can be understood as a similarity or matching metric between context vector $\mathbf{c}$ and word embedding $\mathbf{e}_v$. In \modelname{}, the logit is offset by $\epsilon$ times $\mathbf{c}^\top W \mathbf{e}_v$, which is also a bilinear similarity. To transform this \modelname{} to another language model, we need to map the context vectors and word embeddings between word embedding spaces $\mathbf{e}_v = H \mathbf{e}_v'$
\begin{equation}
    \mathbf{c}^\top W \mathbf{e}_v = (H\mathbf{c}')^\top W (H\mathbf{e}_v') = \mathbf{c}'^\top (H^\top WH) \mathbf{e}_v'
\end{equation}
We work by first identifying a linear mapping $H$ from target LM word embeddings to source LM word embeddings. Then, the matrix $H^\top WH$ can be inserted into the target LM as \modelname{}. This is motivated by prior work on the linear mapping between word embeddings from different models~\citep{li2021learning}. Finally, the calculated steering matrix is directly applied to the target LM.
Figure~\ref{fig:transfer_efficiency}(a) shows the performance after we transfer the \modelname{} learned on GPT2-large to LMs of other sizes, ranging from gpt2 (124M) to GPT-J-6B (6B). We can see a uniform improvement in transferred \modelname{}s, with GPT2 and GPT2-medium getting similar scores (0.307 and 0.308) to the best baseline (DExperts).

\section{Conclusions}
\label{sec:conclusions}
In this work, we discover the prevalent phenomenon of word embeddings containing steers for language model generation. We demonstrate the promise and efficacy of \modelname{}, a theoretically grounded, simple, and lightweight approach for the steering of generative language models. \modelname{} can model various styles and achieve comparable or superior performance to baselines in language model detoxification and generation control. \modelname{} also allows for continuous and compositional control and can be transferred to other language models. More importantly, it provides an interpretation of how word embeddings interplay with language model generation. So far, we have only studied output word embeddings, so it is intriguing to ask whether similar phenomena apply to other components, such as input word embeddings and hidden layers.

\section*{Limitations}
One limitation of \modelname{} is that it works on word embeddings and focuses on conditions related to wording. This restricts its capability to deal with more complex tasks, such as syntactic trees or persuasive techniques that involve logical reasoning. Additionally, our model is dependent on word embeddings, so the model cannot work with language model APIs that do not provide direct access to these embeddings.

\section*{Acknowledgement}
This research is partially supported by U.S. DARPA KAIROS Program No. FA8750-19-2-1004, DARPA INCAS Program No. HR001121C0165, U.S. DARPA SemaFor Pro- gram No. HR001120C0123, DARPA MIPs Program No. HR00112290105, and DARPA ITM Program No. FA8650-23-C-7316. The views and conclusions contained herein are those of the authors and should not be interpreted as necessarily representing the official policies, either expressed or implied, of DARPA or the U.S. Government. The U.S. Government is authorized to reproduce and distribute reprints for governmental purposes, notwithstanding any copyright annotation therein.

\bibliography{custom, anthology}

\appendix
\newpage
\section{Broader Impacts}
\label{appsec:broader_impact}

The intended use of this work is to contribute to advancements in fine-grained and efficient control of language generation in AI, with experiments shown on sentiment modulation, political stance adjustment, and language detoxification. We do not aim to create a tool for manipulating public opinion or promoting specific political ideologies, but instead to provide methods for enhancing the reasoning interpretability, and safety of AI language models. Our techniques offer the potential for fine-tuned sentiment adjustment and toxic content mitigation, thereby contributing to more reliable, unbiased, and respectful language generation systems.
We would like to emphasize that on the problem of language model toxicity, we limit our model to modeling detoxification only. This encourages positive and socially beneficial usage of our model as well as general language models.

\section{Formal Statement of Theorem~\ref{thm:informal}}

\paragraph{Hidden Markov Models}(HMMs)~\citep{baum1970maximization} is a widely used framework for analyzing discrete stochastic processes. Because of its generality (modeling arbitrary distributions), intuitiveness, and interpretability (containing a structured state space), it has long been used as a primary choice when modeling language distribution.
is a discrete stochastic process with a set of $n$ states $\mathbf{S}$ and a set of $m$ observations or emissions $\mathbf{O}$, with arbitrary indexing of $\mathbf{S}$ and $\mathbf{O}$. The time step $t=0$ distribution is determined by initial state distribution $s_0\sim\pi$. For each later time step $t\geq 1$, the state transition probabilities are represented by a matrix $\mathbf{T}$, where $T(s, s')=P(s_{t+1}=s' | s_t=s)$ denotes the probability of transitioning from state $s$ to state $s'$. At each time step one observation $o_t$ is emitted, with the emission probabilities represented by a matrix $\mathbf{B}$, with $B(s, o)=P(o_t=o | s_t=s)$.
A sequence of observations can be denoted as $\mathbf{o} = \{o_1, o_2, \ldots, o_T\}$.
The probability distribution over sequences $\mathbf{o}$ then follows formula:
\begin{multline}
    P(o_1, \cdots, o_T; \pi) \\
    = \pi^\top \left(\prod_{t=0}^{T-1} \text{diag}(\mathbf{p}(o_t)) T \right) \mathbf{p}(o_T),
\end{multline}
where $\mathbf{p}(o)$ is a $|\mathcal{S}|$-dim vector indicating $P(o\mid s)$ for all states $s\in\mathcal{S}$.

\paragraph{Language Models} In generative language models, the sequence is generated word-by-word by a conditional probability $P(o_t\mid o_1, \cdots, o_{t-1})$. The common technique to model this probability is to first calculate the inner product between a contextual vector $\mathbf{c}(o_1, \cdots, o_{t-1})$ and word embeddings $\mathbf{E}=(\mathbf{e}_o, \cdots)\in\mathbb{R}^{d\times|\mathcal{O}|}$, namely, $\mathbf{l} = \mathbf{c}(o_1, \cdots, o_{t-1})^\top \mathbf{E}$. Here,  $\mathbf{l}$ is known as the word \textit{logits}, which then usually passes through a softmax operator to get a distribution over words. For simplicity of analysis, in this work, we assume a linear formulation and let conditional probability $P(o_t|o_1,\cdots,o_{t-1})=\mathbf{c}(o_1,\cdots,o_{t-1})^\top \mathbf{e}_{o_t}$.
By the chain rule, multiplying the conditional probabilities will give us the full probability: $\prod_{t=1}^T P(o_t \mid o_1, \cdots, o_{t-1}) = P(o_1, \cdots, o_T)$.
We are then interested in the situation where a language model is good enough to represent an equivalent distribution with HMM.

In this study, we aim to model the influence of \textit{conditions} in text generation. This section describes how we incorporate conditions in HMMs.
Conventionally, people assume a $d$-dimensional state representation $\phi_s$ for every state $s$, and $d$-dimensional $\psi_o$ for each observation $o$, so that they can compute the probabilities $T(s, s')=\phi_s^\top A \phi_s'$, $B(s, o)=\phi_s^\top \psi_o$ and $\pi(s)=\phi_\pi^\top \phi_s$ for some $\phi_\pi$. We also use matrices $\Phi, \Psi$ to denote the stacked representations $\Phi=(\phi_s | s\in\mathcal{S}), \Psi=(\psi_o | o\in\mathcal{O})$. Here we introduce an additional \textit{condition} component in state representations, so that $\phi_s$ can be partitioned into two sub-vectors: $\phi_s = \begin{pmatrix} \phi_{s, \text{semantic}} \\ \phi_{s, \text{condition}} \end{pmatrix}$.
Here $\phi_{s, \text{semantic}}\in \mathbb{R}^{d_s}$ represents the $d_s$-dim semantic information, and $\phi_{s, \text{condition}} \in \mathbb{R}^{d_c}$ the $d_c$-dim condition information related to state $s$.
Then we assume that the transition probability $T(s, s')$ comes from both semantic relations and conditional similarities between $s'$ and $s$: $T(s, s')=\phi_{s, \text{semantic}}^\top A' \phi_{s', \text{semantic}} + \phi_{s, \text{condition}}^\top\phi_{s', \text{condition}}$.

We also make the following assumptions regarding the state representations:

\begin{assumption}
\label{assumption:representation}
    State representations $\phi$ also satisfy the following properties:

    1. Values for each dimension are uniformly normalized to a constant:
    $\forall i \in [1..d], \sum_{s\in\mathcal{S}}\phi_{s,i}^2 = C$.
    
    2. Dimensions are linearly independent:
    $\forall i,j \in [1..d]$ and $i\ne j$, $\sum_{h\in\mathcal{H}}\phi_{h,i}\phi_{h,j} = 0$.

    3. Dimensions are also conditionally independent:
    if $i,j \in [1..d], k\in[d_s+1 .. d]$ are not all the same, $\sum_{s\in\mathcal{S}} \phi_{s,i}\phi_{s,j}\phi_{s,k} = 0$.
\end{assumption}

The validity of the assumption is discussed in Appendix~\ref{appsec:assumptions}.
Then, we present the result below, revealing that shifting between conditions is equivalent to a linear transformation in word embedding space.

\begin{theorem}
\label{thm:switching_formal}
Assume assumption \ref{assumption:representation} holds. Suppose there are two initial distributions $\pi=\phi_\pi^\top \Phi, \pi'=\phi_{\pi'}^\top \Phi$, so that $\phi_\pi$ and $\phi_{\pi'}$ only differ in their condition-parts: $\phi_{\pi, \text{semantic}} = \phi_{\pi', \text{semantic}}$. Also, suppose the elements in $\phi_{\pi, \text{condition}}$ are non-zero. Then there exists a matrix $W$ so that, by transforming word embeddings from $E$ to $WE$, the LM which originally simulates the text distribution starting with $\pi$ will now turn to be equivalent to a distribution initiating from $\pi'$.
\end{theorem}

\section{Formal Statement and Proof of Theorem \ref{thm:switching_formal}}
\label{appsec:proof}

To prove Theorem~\ref{thm:switching_formal}, we start by claiming a construction of matrix $W$. Then we prove that when assumptions \ref{assumption:representation} hold, $W$ can change each conditional likelihood function from $p(v_i \mid o_1, \cdot, o_{i-1}, \pi)$ to $p(v_i \mid o_1, \cdot, o_{i-1}, \pi')$ up to a constant factor. Finally, by chaining the conditional likelihoods, we see that $W$ can change the sentence-level probability distribution of the HMM from $\pi$-initialization to $\pi'$-initialization.

Assuming full column-rank for $\mathbf{E}$ and $\mathbf{p}(o)$, we have the following connection between LM and HMM:

\begin{proposition}
There exist projection matrices $R_1$ and $R_2$ so that $R_1^\top R_2 = I_n$ and 
\begin{align*}
    &\mathbf{c}(o_1, \cdots, o_{t-1})^\top \\
    =& \left(\frac{\pi^\top \prod_{t'=1}^{t-1} \text{diag}(\mathbf{p}(o_t')) T}
    {\pi^\top \left(\prod_{t'=1}^{t-2} \text{diag}(\mathbf{p}(o_t')) T  \right)\mathbf{p}(o_{t-1})}\right)
    R_1^\top, \\
    &\mathbf{e}_o = R_2 \mathbf{p}(o).
\end{align*}
\end{proposition}

We first construct a helper matrix
$W' = \begin{pmatrix}
    I_{d_s} & 0 \\
    0 & \Lambda
\end{pmatrix}$
so that $\Lambda$ is diagonal and $W'\phi_\text{init} = \phi_\text{init}'$. Such a solution exists as we assume $\phi_{\text{init}, \text{condition}}$ contains only non-zero values.
Then, we can construct the matrix $W$ as
$W=R_1^+ \Phi^\top W' \Phi {R_2^+}^\top$, where $R_1^+, R_2^+$ are pseudo-inverse matrices of $R_1, R_2$, respectively.


\begin{lemma}
$T\Phi^\top W' \Phi = \Phi^\top W' \Phi T$.
\label{lemma:swap_T}
\end{lemma}
\begin{proof}

First, it is easy to see that, by Assumption~\ref{assumption:representation}.1 and Assumption~\ref{assumption:representation}.2, the representation matrix $\Phi$ is row-orthonormal to constant $C_2$:
\[
    \Phi\Phi^T = C_2 I_d
\].

Then we have the following proof:
\begin{align*}
    T \Phi^\top W' \Phi = 
    & \Phi^\top
    \begin{pmatrix}
        T_s & 0 \\
        0 & I_{d_c}
    \end{pmatrix} \Phi \Phi^\top W' \Phi \\
    =& C_2 \Phi^\top
    \begin{pmatrix}
        T_s & 0 \\
        0 & I_{d_c}
    \end{pmatrix}
    \begin{pmatrix}
        I_{d_s} & 0 \\
        0 & \Lambda
    \end{pmatrix} \Phi \\
    =& C_2 \Phi^\top
    \begin{pmatrix}
        T_s & 0 \\
        0 & \Lambda
    \end{pmatrix} \Phi \\
    =& C_2 \Phi^\top
    \begin{pmatrix}
        I_{d_s} & 0 \\
        0 & \Lambda
    \end{pmatrix}
    \begin{pmatrix}
        T_s & 0 \\
        0 & I_{d_c}
    \end{pmatrix} \Phi \\
    =& \Phi^\top W' \Phi \Phi^\top
    \begin{pmatrix}
        T_s & 0 \\
        0 & I_{d_c}
    \end{pmatrix} \Phi \\
    =& W'T
\end{align*}
\end{proof}


\begin{lemma}
\label{lemma:swap_P}
$\forall v\in V$, we have that, $ \Phi \text{diag}(\mathbf{p}(o))\Phi^\top W' \Phi = W' \Phi\text{diag}(\mathbf{p}(o)) $.
\end{lemma}
\begin{proof}
To prove this, we first prove that $\Phi \text{diag}(\mathbf{p}(o))\Phi^\top$ has the form
$\begin{pmatrix}
    A & 0 \\
    0 & \Lambda'
\end{pmatrix}$, where $\Lambda'$ is also diagonal. This is equivalent to saying that, for any two one-hot vectors $\mathbf{e}(i),\mathbf{e}(j)$, if $i\in[d_s+1 .. d]$ or $j\in[d_s+1 .. d]$, then
\begin{multline}
\mathbf{e}_i^\top \Phi ~\text{diag}(\mathbf{p}(o))\Phi \mathbf{e}_j^\top \\
= \sum_{h\in\mathcal{H}} \phi_{h,i}\phi_{h,j}p(v \mid h) = f_v(i,j) \mathbf{1}(i=j).
\end{multline}

For any $i\ne j$,
\begin{align*}
    &\sum_{h\in\mathcal{H}} \phi_{h,i}\phi_{h,j}p(v \mid h)\\
    =& \sum_{h\in\mathcal{H}} \phi_{h,i}\phi_{h,j} \sum_{k\in[1..d]} \phi_{h,k} \theta{v, k} \\
    =& \sum_{k\in[1..d]} \theta{v, k} \sum_{h\in\mathcal{H}} \phi_{h,i}\phi_{h,j}\phi_{h,k}\\
    =& \sum_{k\not \in \{i,j\}} \theta_{v, k} \sum_{h\in\mathcal{H}} \phi_{h,i}\phi_{h,j}\phi_{h,k} \\
    &+ \theta_{v, i} \sum_{h\in\mathcal{H}} \phi_{h,i}^2\phi_{h,j} \\
    &+ \theta_{v, j} \sum_{h\in\mathcal{H}} \phi_{h,i}\phi_{h,j}^2 \\
    \text{(Asm. \ref{assumption:representation}.3)} =& 0
    + \theta_{v, i} \sum_{h\in\mathcal{H}} \phi_{h,i}\phi_{h,j} \\
    \text{(Asm. \ref{assumption:representation}.2)} =& 0 + 0 \\
    =& 0
\end{align*}

We then have the following:
\begin{align*}
    &\Phi \text{diag}(\mathbf{p}(o))\Phi^\top W' \Phi \\
    =& \begin{pmatrix}
    A & 0 \\
    0 & \Lambda'
    \end{pmatrix} W' \Phi \\
    =& \begin{pmatrix}
    A & 0 \\
    0 & \Lambda'\Lambda
    \end{pmatrix} \Phi \\
    =& W' \begin{pmatrix}
    A & 0 \\
    0 & \Lambda'
    \end{pmatrix} \Phi \\
    =& W' \Phi \Phi^\top \begin{pmatrix}
    A & 0 \\
    0 & \Lambda'
    \end{pmatrix} \Phi \\
    =& W' \Phi\text{diag}(\mathbf{p}(o))
\end{align*}
\end{proof}


By combining Lemma~\ref{lemma:swap_T} and ~\ref{lemma:swap_P}, we have the following lemma:
\begin{lemma}  
\label{lemma:swap_TP}
\[
    T\text{diag}(\mathbf{p}(o))\Phi^\top W' \Phi
    = \Phi^\top W' \Phi T\text{diag}(\mathbf{p}(o))
\]
\end{lemma}
\begin{proof}
\begin{align*}
    &T\text{diag}(\mathbf{p}(o))\Phi^\top W' \Phi \\
    =& \Phi^\top \begin{pmatrix}
        T_s & 0 \\
        0 & I_{d_c}
    \end{pmatrix} \Phi \text{diag}(\mathbf{p}(o))\Phi^\top W' \Phi \\
    =&  \Phi^\top \begin{pmatrix}
        T_s & 0 \\
        0 & I_{d_c}
    \end{pmatrix}  W' \Phi \text{diag}(\mathbf{p}(o)) \\
    =&  \Phi^\top 
    W' \Phi T \text{diag}(\mathbf{p}(o)) \\
\end{align*}
\end{proof}


Finally, when we apply Lemma~\ref{lemma:swap_T} and \ref{lemma:swap_TP} to the language model formulation, we can see that the conditional likelihood function has been steered to:

\begin{align*}
    &\mathbf{p}_W(v_i \mid o_1, \cdots, o_{i-1} ; \pi) \\
    =& \mathbf{c}(o_1, \cdots, o_{i-1}; \pi) W E\\
    =& \frac{\pi^\top T \text{diag}(o_1)T\cdots T \text{diag}(o_{i-1}) T R_1^\top}{\pi^\top T\text{diag}(o_1)T\cdots T \mathbf{p}(o_{i-1})} W R_2 P_O \\
    =& \frac{\pi^\top T \text{diag}(o_1)T\cdots T \text{diag}(o_{i-1}) T \Phi^\top W' \Phi}
    {\pi^\top T\text{diag}(o_1)T\cdots T \mathbf{p}(o_{i-1})} P_O \\
    &\text{(Lemma~\ref{lemma:swap_T})} \\
    =&  \frac{\pi^\top T \text{diag}(o_1)T\cdots T \text{diag}(o_{i-1}) \Phi^\top W' \Phi T}
    {\pi^\top T\text{diag}(o_1)T\cdots T \mathbf{p}(o_{i-1})} P_O \\
    &\text{(by Lemma~\ref{lemma:swap_TP})}
    \\
    =&  \frac{\pi \Phi^\top W' \Phi ^\top T \text{diag}(o_1)T\cdots T \text{diag}(o_{i-1}) T}
    {\pi^\top T\text{diag}(o_1)T\cdots T \mathbf{p}(o_{i-1})} P_O \\
    =&  \frac{\phi_{\text{init}} W' \Phi T \text{diag}(o_1)T\cdots T \text{diag}(o_{i-1}) T}
    {\pi^\top T\text{diag}(o_1)T\cdots T \mathbf{p}(o_{i-1})} P_O \\
    =&  \frac{\phi_{\text{init}}' \Phi T \text{diag}(o_1)T\cdots T \text{diag}(o_{i-1}) T}
    {\pi^\top T\text{diag}(o_1)T\cdots T \mathbf{p}(o_{i-1})} P_O \\
    & \text{(by definition)} \\
    =& \frac{\pi' T \text{diag}(o_1)T\cdots T \text{diag}(o_{i-1}) T}
    {\pi^\top T\text{diag}(o_1)T\cdots T \mathbf{p}(o_{i-1})} P_O \\
    &\propto \mathbf{c}(o_1, \cdots, o_{i-1}; \pi') E \\
    =& \mathbf{p}(o_i \mid o_1, \cdots, o_{i-1}; \pi') \\
\end{align*}

Therefore, the steered conditional likelihood is equivalent to an HMM initiating from $\pi'$ (up to a normalization constant over vocabulary $\mathcal{O}$). By chaining the conditional likelihood functions, it is easy to see that the actual output sequence distribution is now:

\begin{align*}
    &p_{W, \text{normalized}}(o_1, \cdots, o_{L} ; \pi) \\
    =& \prod_{i=1}^L \text{normalize}_\mathcal{O}(\mathbf{p}_W(v_i \mid o_1, \cdots, o_{i-1}; \pi)) \\
    =& \prod_{i=1}^L \mathbf{p}(o_i \mid o_1, \cdots, o_{i-1}; \pi') \\
    =& p(o_1, \cdots, o_{L} ; \pi')
\end{align*}

This concludes our proof to Theorem~\ref{thm:switching_formal}.
\section{Implementation Details}
\label{appsec:implementation}

In this paper, we leverage the HuggingFace package\footnote{\url{https://huggingface.co}} and its model checkpoints. To implement \modelname{}, we simply wrap the \texttt{self.forward} function of language model transformer's \texttt{lm\_head}, and inject the computation formula of \modelname{}. In specific, the token logits are replaced w ith $\mathbf{c}^\top (I+\epsilon W) \mathbf{e}_o$, we change the computation order and first compute $\mathbf{c}' = \mathbf{c} +\epsilon W\mathbf{c})$, then compute $\mathbf{c}'^\top \mathbf{e}_o$. We find that this increases computational efficiency in practice and avoids the problem caused by many Transformers sharing input and output word embedding parameters in storage. Another trick we applied in experiments is that, as there is a systematical distribution shift between pre-training corpus and domain-specific dataset (such as detoxification dataset and reviews), we add one ``dummy'' steer $W_\text{dummy}$ to fill this overall distribution gap. Therefore, for positive label training, we use model $P_{\epsilon_0 (W+W_\text{dummy})}$, and for negative label training, we use model $P_{\epsilon_0 (-W+W_\text{dummy})}$. This is where the 3M parameters come from in Table~\ref{tab:efficiency}.

For optimization, we use Adam optimizer~\cite{kingma2014adam} with a 1e-2 learning rate and train for 1k steps. The steer matrix $W$ is initialized with a Gaussian distribution of 0 mean and 1e-3 variance. Across all experiments, we run three initial seeds of 0, 1, and 2 for training. When required to generate 25 sentences on each prompt, we use random seeds 0, 1, 2, ..., 24.  Our hardware is one single Tesla V-100 GPU with 16GB CUDA memory. 
\section{Hyperparameter Selection}
\label{appsec:ablation}

We select the decoding hyper-parameter based on a balance of task performance and generation quality on the detoxification task's dev set. The scores are listed in the table below.

\begin{table*}[!ht]
\centering
\linespread{1}

\begin{tabular}{lcccccc}
\toprule

\multirow{2}{*}{\textbf{steering value} $\epsilon$} & \multicolumn{2}{c}{\textbf{Toxicity}$\downarrow$} & \textbf{Fluency}$\downarrow$ & \multicolumn{2}{c}{\textbf{Diversity}$\uparrow$} \\
 & Avg. max. toxicity & Toxicity prob. & Output ppl. & Dist-1 & Dist-2 & Dist-3 \\

\midrule

GPT-2 (original)
& 0.527 & 0.520 & 25.45 & 0.58 & 0.85 & 0.85 \\

MuCoLa
& 0.308 & 0.088 & 29.92 & 0.55 & 0.82 & 0.83 \\

\midrule

\modelname{} ($\epsilon_0$)
& 0.542 & 0.560 & 24.20 & 0.54 & 0.85 & 0.86 \\

\modelname{} ($2\epsilon_0$)
& 0.473 & 0.388 & 24.54 & 0.54 & 0.84 & 0.85 \\

\modelname{} ($3\epsilon_0$)
& 0.419 & 0.278 & 24.83 & 0.54 & 0.84 & 0.85 \\

\modelname{} ($4\epsilon_0$)
& 0.393 & 0.232 & 25.43 & 0.54 & 0.83 & 0.84 \\

\modelname{} ($5\epsilon_0$)
& 0.370 & 0.198 & 26.37 & 0.54 & 0.83 & 0.84 \\

\modelname{} ($6\epsilon_0$)
& 0.343 & 0.172 & 27.53 & 0.54 & 0.82 & 0.83 \\

\modelname{} ($7\epsilon_0$)
& 0.320 & 0.138 & 29.12 & 0.54 & 0.81 & 0.82 \\

\modelname{} ($8\epsilon_0$)
& 0.306 & 0.118 & 31.32 & 0.54 & 0.80 & 0.81 \\

\bottomrule
\end{tabular}

\label{tab:ablation_epsilon}
\caption{Results on language model detoxification task dev set by selecting different steering value $\epsilon$.}

\end{table*}

When gradually increasing the steering value, the detoxification success rate increases while generation fluency decreases. To better balance the two ends, we select $5\epsilon_0$ for downstream evaluation, as it does not compromise perplexity too much while achieving decent task performance.
\section{Details of Transferring \modelname{} to Other Language Models}
\label{appsec:transfer_details}

To transfer an \modelname{} from one LM $M_1$ to another LM $M_2$, we notice that \modelname{} essentially adds one term $\mathbf{c}^\top W \mathbf{e}_o$ to the logits, where both $\mathbf{c}$ and $\mathbf{e}_o$ can be viewed as residing in word embedding space. Therefore, $W$ can be considered as a similarity matrix in $M_1$'s word embedding space. To use $W$ in $M_2$, we propose to map $M_2$'s word embedding space to that of $M_1$ before using $W$ as usual. The process works in 2 steps.

First, we identify a linear mapping from $W_2$ to $W_1$'s word embedding space. We start with building a list of anchor words. Specifically, we select the top 4k words shared by both vocabularies. We denote the token embedding matrices as $E_1', E_2'$ respectively. Then, we initialize a mapping $H$ with a Gaussian distribution of 1e-3 initial variance, and we apply Adam optimizer 0.01 learning rate for 5k steps. Secondly, After acquiring the mapping matrix $H$, we map both the context and embedding vectors to $H\mathbf{c}$ and $H\mathbf{e}_o$, respectively. So the additive term for LM $M_2$ is now $\mathbf{c}^\top H^\top WH \mathbf{e}_o$, which is equivalent to using a steer matrix $ H^\top WH$ for the LM $M_2$.

This mapping process is not precise, as word embeddings between LMS are not linearly associated.
So we observe an increased instability in generation if we use large $\epsilon$. Therefore, we reduce the steering value to 0.1 of its original scale, $0.5\epsilon_0$ for generation. This is the setting for getting results in Figure~\ref{fig:transfer}.
\section{Details of Investigating Interpretability}
\label{appsec:interpretability_details}

In Section~\ref{subsec:interpretability}, we interpret the weights learned in \modelname{} and list discovered keywords in Table~\ref{tab:interpretability}. A detailed description of getting these results is as follows. First, we conduct SVD decomposition of steer matrix $W$. The resulting $D$ matrix can then be interpreted as a ranked list of significant row vectors. We take the first 9 rows and compute their dot products with word embeddings.
As the row vector does not tell us which of the 2 directions indicates an increased probability, we select 20 tokens with top dot product and 20 tokens with bottom dot product as two candidate groups. Each group is concatenated to a text sequence and passed to Perspective API, and the group with a larger toxicity value is considered true ``keywords''. If, however, Perspective API recognizes the language as not English, which happened to rows No. 2, 4, and 6, then we discard this row as they contain mostly symbols and non-English words. Finally, we filter out suffix tokens, and the remaining keywords are listed in Table~\ref{tab:interpretability}.
\section{Validity of Assumptions}
\label{appsec:assumptions}

To verify the validity of the assumptions, we did an experiment for searching for valid HMMs while satisfying the assumption~\ref{assumption:representation}. It is trivial to construct valid $\Psi$ as long as a valid $\Phi$ can be found. So specifically, we set $d_s=20$ and $d_c=1$ to represent a one-conditional HMM. We let $n=200$ and randomly initialized $\Phi$ with Gaussian distribution with variance 1e-3. Then we construct the following objective function 
\begin{align*}
    \mathcal{L} = \mathcal{L}_\text{norm} + \mathcal{L}_\text{dist} + \mathcal{L}_\text{independence} +
    \mathcal{L}_\text{conditional},
\end{align*}

where 
\begin{align*}
    \mathcal{L}_\text{norm} =& \sum_i ( \sum_{s\in\mathcal{S}} \phi_{s, i}^2 - \frac{1}{dn} \sum_{s\in\mathcal{S}, i'} \phi_{s, i'}^2 )^2 \\
    \mathcal{L}_\text{dist} =& \sum_{s, s'} \max(-T(s, s'), 0) \\
    & +\sum_{s} (\sum_{s'} T(s, s') - 1)^2 \\
    \mathcal{L}_\text{independence} =& \sum_{i\ne j} (\sum_{s} \phi_{s, i} \phi_{s, j})^2 \\
    \mathcal{L}_\text{conditional} =& \sum_{i,j,k \text{not one value}, k\in[d_c+1, d]} \\
    & +(\sum_{s} \phi_{s, i} \phi_{s, j} \phi_{s, k})^2
\end{align*}

Generally, this objection characterizes the derivation of $\Phi$ from the assumptions. We use Adam optimizer with learning rate 1e-3, and ReduceLROnPlateau~\footnote{\url{https://pytorch.org/docs/stable/generated/torch.optim.lr_scheduler.ReduceLROnPlateau.html}} with patience 100 and reduce factor 0.5. The optimization process lasts 500k steps, starting from random seeds 0, 1, 2, and 3. On all random seeds, the objective function reduces from greater than 1 to less than 1e-5. This indicates that valid HMM solutions satisfy the assumption.
\section{Comparison with A (Soft) Word Blacklist}

\label{appsec:blacklist}

First, we explain that a control vector is equivalent to a SWB. This is because by adding a vector to context $c' = c +  \epsilon w$, we are equivalently adding a logit bias to each word: $c'^\top E = c^\top E + \epsilon w^\top E$, where $w^\top E$ is the static logit bias vector for each word. Then we point out that theoretically, \modelname{} is more expressive in representing sequence distributions than them, since \modelname{}’s formulation $c^\top (I+\epsilon W) E$ can let different words be preferred in different contexts $c$. Theoretically, there exists a \modelname{} for steering between any two finite-length distributions (with proof below). SWB cannot achieve this (a counterexample below). Intuitively speaking, a blacklist or whitelist uniformly applied at all positions cannot possibly achieve flexible control over the complex language distribution without hurting the generation quality. 

\subsection{Formal statement of the universality of \modelname{}}
Let $D_1$ and $D_2$ be two finite-length finite-vocabulary sequence distributions, there exists a context vector function $c(o_1, o_2, \cdots, o_i)$, a word embedding $E$ and a matrix $W$ so that an \modelname{} with $-W$ and $+W$ represents distribution $D_1$ and $D_2$ respectively.

\begin{proof}
We prove the existence by construction. Let $I(o_1, o_2, \cdots, o_i) \in \mathbb{N}$ be an arbitrary indexing function for all subsequences. With bounded subsequence length, $I$ values are also bounded by a number $d$. We let $c(o_1, \cdots, o_i) = \sum_o \text{onehot}(I(o_1, o_2, \cdots, o_i, o)) \in \mathbb{R}^d$. For each token $o$, the word embeddings is $e_o$ such that, for any subsequence $(o_1, o_2, \cdots, o_i)$, ${e_o}[I(o_1, o_2, \cdots, o_i, o))] = \frac{(D_1+D_2)(o | o_1, o_2, \cdots, o_i)}{2}$. Then $W$ is as follows: for any $(o_1, o_2, \cdots, o_i)$ and token $o$, $W[I(o_1, o_2, \cdots, o_i, o), I(o_1, o_2, \cdots, o_i, o)] = \frac{(D_2-D_1)(o | o_1, o_2, \cdots, o_i)}{(D_1+D_2)(o | o_1, o_2, \cdots, o_i)}$. It is diagonal. We omit verification due to space limits.
\end{proof}

\subsection{3.2 Construction of a counterexample for SWB}

Let $D_1$ be a one-point distribution on sequence ``AB'', and $D_2$ be a uniform distribution on sequences {``BA'', ``AB''}. For any language model representing $D_1$, there does not exist an SWB that can convert the language model into $D_2$.
Verification of this counterexample is trivial, and we omit it due to space limits.

\section{Results of \modelname{} on Pythia Family}
\label{appsec:pythia}

Pythia~\footnote{\url{https://github.com/EleutherAI/pythia}} is a family of causal language models developed by EleutherAI. Raining in size from 14M to 2.8B, these models provide an excellent testbed for evaluating the effect of language model sizes. The table below shows the performance of \modelname{} applied to the Pythia language models. We can see a trend of better fluency but decreasing detoxification when the model size increases, indicating a greater controlling difficulty and better base quality of larger language models. 

\begin{table*}[ht!]
\centering
\linespread{1}

\begin{tabular}{lcccccc}
\toprule

\multirow{2}{*}{\textbf{Model}} & \multicolumn{2}{c}{\textbf{Toxicity}$\downarrow$} & \textbf{Fluency} & \multicolumn{2}{c}{\textbf{Diversity}$\uparrow$} \\
 & Max. toxicity & Toxicity prob. & Output ppl.$\downarrow$ & Dist-1 & Dist-2 & Dist-3 \\

\midrule

Pythia-14M & 0.208 & 0.04 & 85.67 & 0.50 & 0.84 & 0.86 \\
Pythia-70M & 0.213 & 0.06 & 54.84 & 0.55 & 0.86 & 0.87 \\
Pythia-160M & 0.223 & 0.07 & 35.24 & 0.54 & 0.86 & 0.87 \\
Pythia-410M & 0.255 & 0.13 & 36.71 & 0.58 & 0.86 & 0.86 \\
Pythia-1B & 0.286 & 0.17 & 31.34 & 0.56 & 0.85 & 0.86 \\
Pythia-1.4B & 0.289 & 0.15 & 31.63 & 0.58 & 0.86 & 0.86 \\
Pythia-2.8B & 0.328 & 0.17 & 32.90 & 0.51 & 0.81 & 0.85 \\

\bottomrule

\vspace{-4mm}

\end{tabular}

\caption{Language model detoxification results of \modelname{} with Pythia}
\vspace{-1mm}
\label{tab:pythia}
\end{table*}
\section{LoRA Configuration}

Thanks for suggesting a more comprehensive evaluation. We use the Huggingface Transformers’ default implementation of LoRA on GPT-2-large to align with standard practices in the field and compare with our method. To test the effect of configurations, we iterate the LoRA rank over 8, 16, 32, and 64. Following the practice of \cite{lee2023platypus}, we set the alpha scalar equal to the rank. We report LoRA’s performance below.

\begin{table*}[ht]
\centering
\begin{tabular}{ccccccc}
\toprule
LoRA rank & Max. Toxicity & Toxicity prob. & ppl. & Dist-1 & Dist-2 & Dist-3 \\
\midrule
8 & 0.362 & 0.257 & 23.83 & 0.532 & 0.839 & 0.852 \\
16 & 0.365 & 0.210 & 21.11 & 0.534 & 0.845 & 0.855 \\
32 & 0.351 & 0.229 & 26.13 & 0.529 & 0.840 & 0.853 \\
64 &0.354 & 0.257 & 23.78 & 0.531 & 0.840 & 0.853 \\
\bottomrule
\end{tabular}
\caption{The performance of LoRA with different ranks on the detoxification dataset.}
\label{tab:lora_rank}
\end{table*}
\section{Incorporating in LoRA}
\label{appsec:plus_lora}
Thanks to its theoretical foundations, \modelname{} is orthogonal to other methods and is intuitively compostable with other control methods. As an example study, we select LoRA to combine with \modelname{} on detoxification with GPT2-large as the backbone. The result is as follows. Combining LoRA with \modelname{} produces a better detoxification score than LoRA alone (although not as good as \modelname{} alone), at the cost of a degraded quality score.

\begin{table*}[h]
\centering
\begin{tabular}{ccccccc}
\toprule
\textbf{Method} & \textbf{Max Toxicity} & \textbf{Toxicity prob} & \textbf{ppl} & \textbf{Dist-1} & \textbf{Dist-2} & \textbf{Dist-3} \\
\hline
LoRA & 0.365 & 0.210 & 21.11 & 0.53 & 0.85 & 0.86 \\
\modelname{} & 0.238 & 0.078 & 28.26 & 0.55 & 0.84 & 0.84 \\
LoRA + \modelname{} & 0.258 & 0.086 & 41.78 & 0.54 & 0.82 & 0.83 \\
\bottomrule
\end{tabular}
\caption{Results on incorporating \modelname{} with LoRA.}
\label{tab:incorporate-lora}
\end{table*}
\section{Results on GPT-J-6B}
\label{sec:gpt-j}
We go on and evaluate \modelname{} on GPT-J-6B model. We use the same evaluation setup as in the main body. The results are shown in Table~\ref{tab:gptj}. We observe that \modelname{} is able to reduce the toxicity of the generations while maintaining the perplexity and the fluency of the generations. We also observe that the diversity of the generations is not affected by \modelname{}.

\begin{table*}[h]
\centering
\begin{tabular}{lcccccc}
\toprule
\textbf{Method} & \textbf{Max Toxicity} & \textbf{Toxicity prob} & \textbf{ppl} & \textbf{Dist-1} & \textbf{Dist-2} & \textbf{Dist-3} \\
\midrule
GPT-J-6B & 0.364 & 0.229 & 18.70 & 0.55 & 0.84 & 0.85 \\
GPT-J-6B + \modelname{} & 0.265 & 0.124 & 18.26 & 0.54 & 0.84 & 0.85 \\
\bottomrule
\end{tabular}
\caption{Results of \modelname{} on GPT-J-6B model on language model detoxification.}
\label{tab:gptj}
\end{table*}

\section{Effect of \modelname{} on Instruction Following}
\label{appsec:plus_instruction}
We study the effect of \modelname{} on the LMs' performance under prompts. For prompts, we let ChatGPT write 10 strongly positively biased prompts in Table~\ref{tab:instructions}.
Then, we let GPT-J generate 25 tokens under instructions (i.e., prompting) and compare its generation with and without \modelname{}.

\begin{table*}[h]
\centering
\begin{tabular}{p{140mm}}
\toprule
``Compose a critique of the novel 'Pride and Prejudice', emphasizing its flaws while maintaining an overall appreciative tone.`` \\
\hline
``Write a review of the restaurant 'Le Gourmet' focusing on areas for improvement, yet with an underlying tone of admiration for its cuisine.`` \\
\hline
``Pen a critique about the summer season, pointing out its drawbacks but in a way that overall celebrates its beauty and warmth.`` \\
\hline
``Draft a review for the TV show 'Breaking Bad', highlighting its weaker aspects but still expressing overall enthusiasm for the series.`` \\
\hline
``Create a review of the coffee shop 'Bean There', discussing its shortcomings while still conveying a sense of enjoyment of its atmosphere.`` \\
\hline
``Write an appraisal of the classic album 'Abbey Road' by The Beatles, noting any perceived faults but with a tone that remains reverent of its musical genius.`` \\
\hline
``Offer a critique of the play 'Hamilton', focusing on its less successful elements while still acknowledging its groundbreaking impact.`` \\
\hline
``Compose a review of the city of Paris in winter, pointing out the challenges of the season while still capturing the magic of the city during this time.`` \\
\hline
``Draft a review of the novel '1984' by George Orwell, discussing its more controversial or challenging aspects but in a context of overall admiration.`` \\
\hline
``Write an evaluation of the gaming console PlayStation 5, noting its limitations or flaws while still expressing enthusiasm for its technological advancements.`` \\
\bottomrule
\end{tabular}
\caption{Instructions proposed by ChatGPT.}
\label{tab:instructions}
\end{table*}

We see that \modelname{} can still steer generation even under a positive prompt while maintaining generation fluency.
Without \modelname{}, the sentiment is 94.00.
With a positive \modelname{}, the average sentiment increases to 97.20.
With negative \modelname{} steering, the average sentiment decreases significantly to 82.80. Example generation are listed in Table~\ref{tab:instructions_gen_vanilla} and \ref{tab:instructions_gen_steered}.

\begin{table*}[t]
\centering
\begin{tabular}{p{140mm}}
\toprule
``Focus of your critique on its plot and characterization rather than its historical accuracy and literary style.''\\
\hline
``(For example: not enough bread, or service was bad, or the bathrooms are dirty, or ...'' \\
\hline
``In the early weeks of the season, most of the country was in the grip of an intense summer.'' \\
\hline
``This is a good show for all the time it spends on drug dealing but the characters are'' \\
\bottomrule
\end{tabular}
\caption{Example generations (one each for the first 4 prompts) without \modelname{}}
\label{tab:instructions_gen_vanilla}
\end{table*}

\begin{table*}[t]
\centering
\begin{tabular}{p{140mm}}
\toprule
`` the novel ` fails to persuade the reader that the characters'in every sense of the verb ...'' \\
\hline
``The rest of the review is up to you. you should probably mention how long you've...'' \\
\hline
``Tt's difficult to redeem the negative factors of summers, except for its  unpredictability.'' \\
\hline
``Dull. It's a cheap improv comedy that's almost utterly incoherent in an...'' \\
\bottomrule
\end{tabular}
\caption{Example generations (one each for the first 4 prompts) with negative \modelname{}}
\label{tab:instructions_gen_steered}
\end{table*}
\section{Embedding Tuning}
\label{appsec:plus_tuning}
We also consider the possibility of tuning the word embeddings of the backbone model. We use the same training procedure as \modelname{}, but instead of tuning the last layer, we tune the word embeddings. We use the same hyperparameters as \modelname{}. The results are shown in Table~\ref{tab:embedding-tuning}. We observe that the performance is comparable to Soft-Blacklist, but is still inferior to \modelname{}. We believe that this is because the word embeddings are shared across all the layers, and tuning the word embeddings may cause the model to forget the knowledge learned from the pre-training. We will leave the investigation of this direction for future work. Embedding-tuning is also more expensive than \modelname{}, as it requires tuning the word embeddings for the entire vocabulary, while \modelname{} only requires tuning the last layer.

\begin{table*}[h]
\centering
\begin{tabular}{ccccccc}
\toprule
\textbf{Method} & \textbf{Max Toxicity} & \textbf{Toxicity prob} & \textbf{ppl} & \textbf{Dist-1} & \textbf{Dist-2} & \textbf{Dist-3} \\
\midrule
Embedding Tuning & 0.289 & 0.0952 & 20.41 & 0.53 & 0.84 & 0.85 \\
\bottomrule
\end{tabular}
\caption{Results on embedding tuning.}
\label{tab:embedding-tuning}
\end{table*}

\end{document}